\documentclass{article}

\PassOptionsToPackage{numbers}{natbib}
\usepackage[preprint]{neurips_2026}

\usepackage[utf8]{inputenc}
\usepackage[T1]{fontenc}
\usepackage[table,xcdraw,usenames,dvipsnames]{xcolor}
\usepackage{hyperref}
\usepackage{url}
\usepackage{booktabs}
\usepackage{amsfonts}
\usepackage{nicefrac}
\usepackage{microtype}
\usepackage{amsmath}
\usepackage{amssymb}
\usepackage{mathtools}
\usepackage{amsthm}
\usepackage{pifont}
\usepackage{enumitem}
\usepackage[most]{tcolorbox}

\usepackage{graphicx}
\usepackage{multirow}
\usepackage[capitalize,noabbrev]{cleveref}
\usepackage{colortbl}
\usepackage{xspace}
\usepackage{siunitx}
\usepackage{caption}
\usepackage{subcaption}
\usepackage{listings}
\usepackage{makecell}
\usepackage{wrapfig}
\newcommand{\ours}{\textit{GAMER}}
\newcommand{\oursn}{GAMER}

\usepackage{amsmath,amsfonts,bm}

\def\eqref#1{equation~\ref{#1}}

\def\1{\bm{1}}

\DeclareMathAlphabet{\mathsfit}{\encodingdefault}{\sfdefault}{m}{sl}
\SetMathAlphabet{\mathsfit}{bold}{\encodingdefault}{\sfdefault}{bx}{n}

\newcommand{\stitle}[1]{\vspace{1ex}\noindent{\bf #1}}

\newcommand{\alfw}{AlfWorld}
\newcommand{\sciw}{SciWorld}
\newcommand{\pddl}{PDDL}
\newcommand{\tool}{Tool}

\theoremstyle{plain}
\newtheorem{theorem}{Theorem}[section]

\theoremstyle{definition}

\newtheorem{assumption}[theorem]{Assumption}
\theoremstyle{remark}

\title{Bridging Inference-Time Scaling and Episodic Memory with Action-Centric Graphs}

\author{%
  \textbf{Xu Zheng}$^{1}$ \quad
  \textbf{Chaohao Lin}$^{1}$ \quad
  \textbf{Zhuomin Chen}$^{1}$ \quad
  \textbf{Weijieying Ren}$^{2}$\\[-0.2em]
  \textbf{Haifeng Chen}$^{3}$ \quad
  \textbf{Wei Cheng}$^{3}$ \quad
  \textbf{Dongsheng Luo}$^{4}$\thanks{Corresponding author. This work was primarily conducted while he was at Florida International University.}\\[-0.2em]
  $^{1}$Florida International University, Miami, FL, USA\\
  $^{2}$Stanford University, Stanford, CA, USA\\
  $^{3}$NEC Laboratories America, Princeton, NJ, USA\\
  $^{4}$Singapore Management University, Singapore\\
}

\begin{document}

\maketitle

\begin{abstract}
Recent advancements in inference-time scaling have significantly unlocked the complex reasoning capabilities of Large Language Models~(LLMs). However, for agents, these approaches suffer from a critical inefficiency, operating in a stateless manner and engaging in redundant search processes. Existing memory mechanisms largely rely on the reasoning capabilities of LLMs, leading to prohibitive computational costs.
In this paper, we propose a novel framework, {\ours}~(Graph-based Action-centric Memory with Episodic Reasoning), that bridges the gap between inference scaling and episodic memory. Our approach models historical reasoning as a dynamic \textit{Action-Centric Graph}.
By decoupling the memory mechanism from LLMs, our method can save token/money usage by providing less memory context than memory mechanism baselines.
To extract knowledge from the graph effectively, we use a dual-stream Temporal Difference learning mechanism to estimate the positive~(suggestion) and negative~(avoidance) value of action nodes based on past successes and failures.  During the inference phase, this learned value function optimizes decision-making bi-directionally, so that positive values provide action suggestions, while negative values indicate high-risk actions. 
By performing efficient searches on the graph, our method significantly improves the efficiency of inference scaling. Experiments on multiple benchmarks demonstrate that {\ours} achieves superior performance by \textbf{20.81\%/6.17\%} for success/progress rate compared to vanilla baselines.
\end{abstract}

\section{Introduction}

As Large Language Models (LLMs), such as ChatGPT~\cite{chatgpt}, Gemini~\cite{team2024gemini}, and Llama~\cite{touvron2023llama}, continue to redefine the frontier of artificial intelligence, LLM-driven agents have exhibited unprecedented prowess in planning~\cite{erdogan2025planandact}, reasoning~\cite{wang2022scienceworld}, and action execution~\cite{NEURIPS2024_b631da75}. When tackling complex downstream domains, ranging from code generation~\cite{zhang2024codeagent} to embodied artificial intelligence~\cite{ma2024survey},
these agents are required to generate long sequences of actions and reason through intricate dependencies. However, the capabilities of standard generation often fall short for such multi-step problems, necessitating a shift towards advanced reasoning strategies~\cite{sel2025llms}. 

Recent research has demonstrated the effectiveness of Inference-time Scaling~\cite{wu2025inference}, where models utilize additional test-time compute to explore potential solutions. Techniques such as Chain-of-Thought~(CoT)~\cite{Wei_COT} and Tree-of-Thoughts~(ToT)~\cite{yao2023tree} explicitly construct reasoning paths, allowing models to decompose problems and verify intermediate steps. Despite their effectiveness, current inference scaling methods suffer from a fundamental limitation that they operate in a stateless manner. Each time an agent encounters a problem, it initiates a search process from scratch~\cite{wang2023selfconsistency}, independent of its past experiences. This leads to significant inefficiency, particularly in persistent environments or recurring tasks. For instance, even when an agent encounters a problem structurally similar to a previously solved instance, standard methods lack the mechanism to leverage prior search trajectories. Consequently, agents redundantly re-explore the massive search space, consuming significant tokens to re-verify unproductive paths or rediscover effective heuristics. This amnesia results in prohibitive computational overhead and high latency~\cite{salama-etal-2025-meminsight}. 

To address the need for persistence, Agent Workflow Memory~(AWM)~\cite{wang2025agent} 
introduces agents to learn reusable task workflows from past experiences, and ReasoningBank~\cite{ouyang2025reasoningbank} distills generalizable reasoning strategies from both successful and failed experiences and abstracts experiences into reusable reasoning units. Although effective, these methods heavily rely on the reasoning capabilities of LLMs and require a large amount of tokens for inference.
Recent memory-augmented frameworks have been proposed, such as A-MEM~\cite{xu2025amem} and G-MEM~\cite{zhang2025gmemory}. While these works substantially enhance long-term information retention, they are fundamentally designed for Context Augmentation rather than inference scaling.
For example, A-MEM optimizes the retrieval of episodic events to overcome context window constraints, and G-MEM structures memory into knowledge graphs to capture entity relationships. These approaches treat memory purely as a retrieval module that fetches declarative facts and injects them into the input context, rather than providing information for computation optimization. Even when highly relevant facts are retrieved, the agent still requires reconstructing the reasoning tree, generating intermediate steps, and evaluating branches ab initio. Crucially, existing memory systems neither prune the search space nor guide the decision-making process, leaving the dominant cost of redundant reasoning unsolved.

In this paper, we bridge this gap by proposing a novel framework, {\ours}~(Graph-based Action-centric Memory with Episodic Reasoning), that decouples the memory mechanism from the reasoning process of the LLMs. By transforming memory from a passive storage of logs into an active \text{Action-Centric Graph}, our method does not require LLM reasoning ability for knowledge extraction. Instead of storing independent linear sequences, we model the reasoning space as a graph where nodes represent discrete actions. To navigate this graph efficiently during inference scaling, we employ Temporal Difference~(TD) learning to iteratively update the value of each action node based on reward signals. Compared with other memory mechanisms, our method has fine-grained knowledge for each task and a low-level token cost. This mechanism enables bi-directional learning from both success and failure trajectories. \ding{182} Positive Guidance: We identify effective actions that contribute to correct solutions and propagate high-value estimates via TD updates, guiding the search algorithm to prioritize these high-probability paths. \ding{183} Negative Pruning: Equally importantly, we learn from invalid or failure trajectories. By assigning low or negative values to actions that lead to dead ends, our memory acts as a ``soft barrier'', effectively diminishing the search space before computational resources are wasted on known failure modes.
Our contributions are summarized as follows: 
\begin{itemize}[itemsep=1.5pt,topsep=0pt,parsep=0pt,leftmargin=*] 
    \item[\ding{72}] We introduce a memory-augmented inference framework that constructs a task-specific Action-Centric Graph, shifting the paradigm from retrieving static text to querying dynamic action values. 
    \item[\ding{72}] We propose a bi-directional update mechanism based on TD learning. This allows the agent to simultaneously learn high-value heuristics from successful attempts and strict constraints from invalid actions, ensuring continuous refinement of the search policy. 
    \item[\ding{72}] We validate our framework through extensive experiments on multiple benchmarks. The results demonstrate that our method consistently outperforms strong baselines. Compared with the vanilla best-of-N method,  {\ours} improves the performance of success/progress rate by \textbf{15.4\%/5.47\%}.
\end{itemize}

\section{Preliminaries}
\label{sec:preliminary}

\stitle{Problem Formulation: Long-Horizon Tasks as POMDPs.}

Following prior work, we model the agent's interaction in a long-horizon task as a Partially Observable Markov Decision Process (POMDP)~\citep{pmlr-v202-carta23a,wang2025steca}. A POMDP is defined by the tuple $\mathcal{M} = (\mathcal{S}, \mathcal{A}, \mathcal{T}, \mathcal{R}, \mathcal{O}, \mathcal{Z}, \gamma)$, where $\mathcal{S}$ is the latent state space, $\mathcal{A}$ is the action space (e.g., text-based thoughts or tool calls), $\mathcal{T}: \mathcal{S} \times \mathcal{A} \rightarrow \mathcal{P}(\mathcal{S})$ is the transition function, $\mathcal{R}: \mathcal{S} \times \mathcal{A} \rightarrow \mathbb{R}$ is the reward function, $\mathcal{O}$ represents the observation space, and $\mathcal{Z}: \mathcal{S} \times \mathcal{A} \rightarrow \mathcal{P}(\mathcal{O})$ is the observation model. 
In our context, the agent does not observe the true state $s_t$ but receives an observation $o_t \in \mathcal{O}$ (e.g., task description or code execution result) governed by $\mathcal{Z}$. A solution is represented as a trajectory $\tau = (o_0, a_0, r_0, \dots, o_T)$, where the goal is to maximize the cumulative reward $R(\tau) = \sum_{t=0}^T \gamma^t r_t$.

\stitle{Inference Scaling}.
Moving beyond standard LLM agent generation, which samples actions autoregressively from a policy $\pi(a_t \mid h_t)$, inference-time scaling methods explicitly construct a search structure to explore the reasoning space at test time. Representative approaches such as ToT~\citep{yao2023tree} formulate inference as a structured search problem rather than a single forward pass.
Formally, at step $t$, a search tree $\Psi$ is constructed, where each node $n \in \Psi$ corresponds to a partial reasoning state~(i.e., a thought prefix), and each edge represents a reasoning expansion. 
The objective is to identify an optimal reasoning trajectory $\tau^*$ that maximizes the cumulative reward. To achieve this, the algorithm utilizes a node-level value function $V(n)$ to guide the search process, estimating the likelihood, leading to a correct final answer. 
Crucially, such thought-level inference scaling operates under a set of implicit assumptions: reasoning steps are reversible, can be arbitrarily resampled, and do not induce side effects on the underlying problem state. As a result, the search tree $\Psi$ is typically constructed \text{tabula rasa} for each new task, with no notion of persistent state or historical reuse.

In contrast, inference-time scaling for agent actions fundamentally differs from thought-level reasoning. 
Standard thought-based inference scaling relies on the assumption that arbitrary intermediate states can be saved and reverted to for tree-based branching. However, interactive agent environments typically lack this mid-state reversibility, restricting agents to generating independent, end-to-end interaction trajectories while receiving delayed and sparse rewards.
As a result, simply extending existing frameworks proves ineffective in dynamic, interactive settings. Motivated by this perspective, we revisit inference-time scaling from an agent-centric perspective. Rather than constructing a reasoning tree from scratch, we leverage historical interaction trajectories to initialize, constrain, and reuse a structured action-centric graph, enabling scalable inference under irreversible dynamics.

\begin{figure*}[t]
    \centering
    \includegraphics[width=1.0\linewidth]{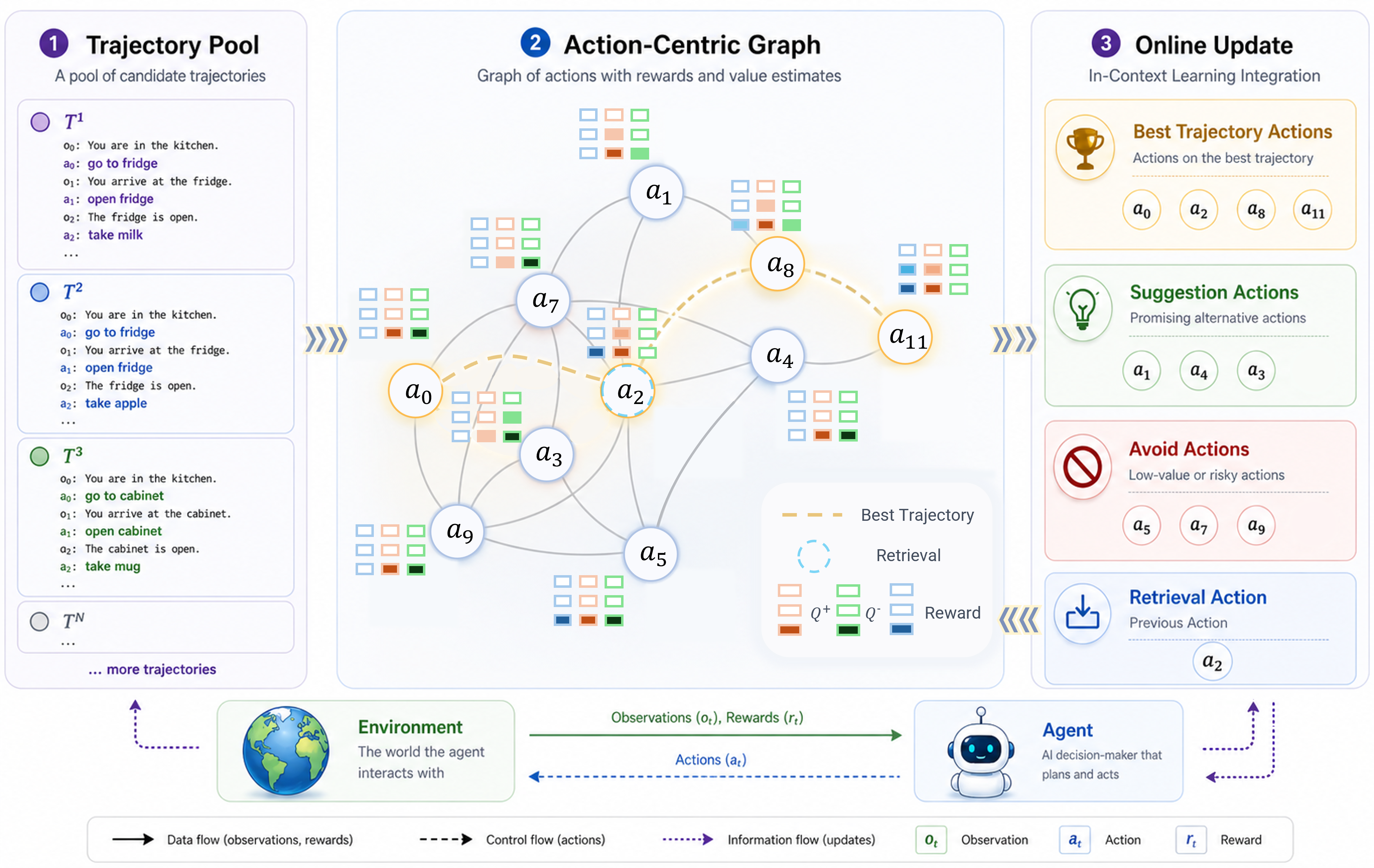}
    \caption{The overall pipeline of our framework, which can be divided into three stages, including 1. trajectory pool collection, 2. trajectory graph building, and 3. online updating. In this pipeline, the top-3 actions are selected.}
    \label{fig:pipeline}
    \vspace{-0.6cm}
\end{figure*}

\section{Methodology}
We aim to overcome the inefficiency of stateless inference scaling by equipping agents with a persistent, evolving memory of reasoning procedures. We propose {\ours}, a framework that constructs an Action-Centric Graph from historical trajectories and utilizes TD learning to estimate the value of reasoning steps.
As illustrated in Figure~\ref{fig:pipeline}, our framework operates in three stages: Graph Construction, Value Learning, and Memory-Guided Inference. 
We first aggregate historical interaction trajectories into a unified graph structure, where nodes represent unique actions and edges represent sequential dependencies.
We then apply TD learning on this graph to estimate the action value from both positive and negative sides ($Q^+, Q^-$), capturing the expected long-term success rate of each action based on historical rewards. 
During the inference phase of a new task, the agent queries this graph to retrieve high-value actions for in-context learning. Furthermore, it leverages identified low-value nodes to act as a soft barrier, dynamically penalizing the LLM's action probability distribution to suppress historically poor choices during independent trajectory generation.

\subsection{Graph Construction}
Rather than relying on traditional memory systems that store raw text or state-observation pairs, we abstract reasoning into a structured \textit{Action-Centric Graph}.
This design is motivated by the observation that while task descriptions vary, the underlying reasoning steps often share structural similarities across tasks.
Formally, given a dataset of historical trajectories $\mathcal{D} = \{\tau_1, \tau_2, \dots, \tau_N\}$, where each trajectory $\tau_i = (o_0, a_0, r_0, \dots, a_T, r_T)$ consists of a sequence of observations, actions, and rewards. We define the Action-Centric Graph as a directed graph $\mathcal{G} = (\mathcal{V}, \mathcal{E})$.

Each node $v \in \mathcal{V}$ represents a distinct action $a$ taken by the agent, which refers to a specific operation executed by the agent. 
We treat unique action representations as distinct nodes, creating a direct mapping from the action space to the graph vertices: $\mathcal{V} = \{ a \mid a \in \tau, \forall \tau \in \mathcal{D} \}$.
A directed edge $e_{ij} = (v_i, v_j) \in \mathcal{E}$ represents a {temporal transition} observed in the history. An edge exists from $v_i$ to $v_j$ if, in any trajectory $\tau \in \mathcal{D}$, action $a_j$ was executed immediately following action $a_i$. This edge encodes the {procedural flow}, which signifies that state $v_i$ has historically transitioned to state $v_j$, effectively capturing the sequential dependencies inherent in the task execution.

By merging trajectories into this graph, $\mathcal{G}$ condenses discrete linear experiences into a global map of the reasoning space. Branching points in the graph represent decision states where the agent previously explored different strategies, and cycles represent recursive reasoning loops.

\subsection{Value Estimation by Dual-Stream TD Learning}
To navigate the \textit{Action-Centric Graph} effectively, the agent requires a robust mechanism to evaluate the quality of each action node. 
Mere connectivity indicates {reachability}, but not {utility}. We formulate this as a value estimation problem, where the goal is to learn a value function $Q(v)$ representing the expected long-term cumulative reward of executing the action at node $v$ and following the policy thereafter.
Since the graph topology evolves dynamically and reasoning trajectories can be lengthy, waiting for the final task completion to update all intermediate nodes is computationally inefficient and suffers from high variance. Instead, we employ \textit{TD Learning}~\cite{sutton1998reinforcement}, a model-free reinforcement learning algorithm that updates current value estimates based on subsequent estimates.

\stitle{TD Learning}. Formally, consider a transition in the graph from node $v_t$ to a successor node $v_{t+1}$ accompanied by an immediate reward $r_t$. The value $Q(v_t)$ is updated iteratively to minimize the \textit{TD error} $\delta_t$, which measures the discrepancy between the current estimate and the target value derived from the next step:
\begin{equation}
    \delta_t = r_t + \gamma \cdot \mathbb{E}[Q(v_{t+1})] - Q(v_t)
\end{equation}
where $\alpha \in (0, 1]$ is the learning rate and $\gamma \in [0, 1]$ is the discount factor controlling the horizon of the value propagation. By iteratively updating $Q(v_t) \leftarrow Q(v_t) + \alpha \delta_t$, the value signal from successful outcomes propagates backward from terminal states to earlier reasoning steps, gradually shaping the graph into an accurate representation of the utility landscape for reasoning.

\stitle{Dual-Stream TD Learning}. While the standard TD formulation provides a single scalar utility, complex reasoning tasks necessitate a more nuanced evaluation. An action might be potentially high-reward but carry a catastrophic risk of failure. Averaging these signals into a single value tends to neutralize critical risk information. 
For instance, given an action $a_{risky}$, suppose this action leads to a perfect solution 50\% of the time~(Reward $+1$) but causes a fatal runtime error 50\% of the time~(Reward $-1$). The expected value is $E[Q] = 0.5(1) + 0.5(-1) = 0$. The agent would treat this high-risk action as equivalent to a completely mediocre action $a_{idle}$ that simply wastes a step~(Reward $0$). The critical risk signal is neutralized by the reward. However, this action should be avoided because of the high risk.

To address this, we extend the standard framework into a \textit{Dual-Stream Value Estimation} mechanism, maintaining two distinct value functions $Q^+(v)$ and $Q^-(v)$ to separately track positive potential and negative constraints.
\begin{itemize}[leftmargin=*]
    \item \textbf{Positive Value $Q^+(v)$:} This function estimates the expected success rate of an action. It is updated using the standard environmental reward $r_t$. $Q^+$ effectively identifies and optimizes the most efficient trajectory toward the goal state.
    
    \item \textbf{Negative Value $Q^-(v)$:} This function estimates the risk of failure. To compute this, we introduce a \textit{Thresholded Reward Redefinition}. Let $\epsilon_{r}$ be a predefined acceptability threshold. We define a binary negative reward $r^-_t$ as:
    \begin{equation}
        r^-_t = 
        \begin{cases} 
        -1 & \text{if } r_t < \epsilon_{r} \\
        0 & \text{otherwise}
        \end{cases}
    \end{equation}
    $Q^-$ explicitly tracks actions that historically led to sub-threshold performance.
\end{itemize}

\stitle{TD Update Rules}. Both value functions are updated by using TD learning. For a transition $(v_t, v_{t+1})$, the updates are performed independently:
\begin{align*}
    Q^{+/-}(v_t) \leftarrow Q^{+/-}(v_t) + \alpha \left[ r^{+/-}_t + \gamma \operatorname*{\max/\min}_{v' \in \text{Succ}(v_t)} Q^{+/-}(v') - Q^{+/-}(v_t) \right]
\end{align*}
Where $\text{Succ}(v_t)$ is the set of successors of $v_t$ in the graph.
Through this dual mechanism, the graph evolves into a nuanced map that guides the agent to maximize success probability while strictly avoiding historical pitfalls.

\subsection{Graph-Memory Guided Inference Scaling}
Unlike rigid search algorithms that explicitly manipulate tree nodes, we leverage the robust \textit{In-Context Learning~(ICL)}~\cite{xie2021explanation} capabilities of LLMs. We translate the learned value signals from our Action-Centric Graph into natural language prompts, guiding the inference process through both global references and local constraints. We assume the maintenance of a trajectory pool $\mathcal{D}_{pool}$ containing historical interaction logs. During the inference phase for a new task, our guidance mechanism operates on two levels:

\stitle{Reference Trajectory Guidance}.
Before the reasoning process begins, we identify a blueprint to guide the overall planning. We retrieve the \textit{Historically Optimal Trajectory} $\tau^*$ from the pool $\mathcal{D}_{pool}$.
\begin{equation}
\tau^* = \operatorname*{argmax}_{\tau \in \text{Retrieve}(\mathcal{D}_{pool})} R(\tau) \quad
\end{equation}
This trajectory $\tau^*$ serves as a high-quality few-shot example in the prompt. It provides the agent with a successful action sequence reference.

\stitle{Dual-Stream Action Injection}. At each specific step $t$ of the inference scaling, the agent needs to decide the next action. Instead of letting the model hallucinate or randomly sample, we inject explicit suggestions derived from the current node $v_t$ in our graph. We query the graph for the successors of $v_t$ and generate two distinct lists based on the dual-stream values:
\begin{itemize}[leftmargin=*]
    \item \textbf{Suggested Actions:} We select the Top-$K$ actions with the highest $Q^+$ values. These represent historically proven strategies for the current state.
    \begin{equation}
        \mathcal{A}_{suggest} = \operatorname*{TopK}_{a \in \text{Succ}(v_t)} \left( Q^+(a) \right)
    \end{equation}
    
    \item \textbf{Actions to Avoid:} We select the Top-$K$ actions with the lowest (most negative) $Q^-$ values. These represent known pitfalls or dead ends.
    \begin{equation}
        \mathcal{A}_{avoid} = \operatorname*{TopK}_{a \in \text{Succ}(v_t)} \left( |Q^-(a)| \right)
    \end{equation}
\end{itemize}

\stitle{Prompt Integration}. These retrieved signals are dynamically injected into the LLM's prompt context at runtime. 
By converting numerical value estimates into semantic instructions, we steer the LLM’s probability distribution toward high-value regions while suppressing error-prone trajectories, improving inference efficiency without modifying the underlying decoding process. Rather than traditional POMDPs that suffer from state aliasing under abstraction, the Action-Centric Graph mitigates this issue through two mechanisms: LLM-generated actions are inherently descriptive and encode rich trajectory information, and LLMs act as an intrinsic contextual filter that evaluates retrieved historical actions in context, assigning near-zero probability to invalid actions and thereby preventing execution of erroneous branches.
The prompt is structured as follows:
\begin{tcolorbox}[title=Integration prompt template]
This is the best historical trajectory for the current task you had, you should first mimic these actions to get a good start:\{$\tau^*$\} \\
Based on successful trajectories for this task, here are some high-value actions you might consider: \{$\mathcal{A}_{suggest}$\}\\
WARNING: Based on analysis of unsuccessful trajectories, you should avoid these actions as they typically lead to poor outcomes: \{$\mathcal{A}_{avoid}$\}
\end{tcolorbox}

\section{Theoretical Analysis}
\label{sec:theory}
In this section, we provide a theoretical guarantee for the efficiency of our framework {\ours}. Specifically, we analyze how the \textit{Dual-Stream Value Estimation} ($Q^+, Q^-$) and prompt-guided inference alter the probability distribution of trajectory rewards, thereby accelerating the convergence of inference scaling.

\stitle{Definitions and Framework.}
Let $\pi_{\theta}$ denote the base policy of the LLM. For a given task, the agent generates a trajectory $\tau$ which yields a scalar reward $R(\tau)$. We treat the reward as a random variable $X$. Let $X_{base} \sim F_{base}(x)$ be the random variable representing the reward obtained by the base policy $\pi_{\theta}$ (without memory), where $F_{base}(x) = P(X_{base} \le x)$ is the cumulative distribution function (CDF). Let $X_{mem} \sim F_{mem}(x)$ be the random variable representing the reward obtained by our memory-guided policy $\pi_{mem}$, which incorporates $\mathcal{A}_{suggest}$ and $\mathcal{A}_{avoid}$ via in-context learning.
In the context of inference scaling, we are interested in the expected maximum reward over $N$ independent samples: $J(N) = \mathbb{E}[\max_{i=1}^N X_i]$. 
Our goal is to show that memory guidance induces First-order Stochastic Dominance (FSD), ensuring $J_{mem}(N) \ge J_{base}(N)$.

\stitle{Inference Scaling Efficiency.} The core intuition of {\ours} is that the \textit{Action-Centric Graph} modifies the sampling distribution in two distinct ways. The injection of $\mathcal{A}_{suggest}$ increases the probability of selecting actions that historically led to high rewards. The warning against $\mathcal{A}_{avoid}$ eliminates the probability of selecting actions that historically led to failures.

\begin{assumption}[Monotonic Probability Re-allocation]
\label{ass:reallocation}
The memory guidance mechanism acts as a transformation operator $T: \pi_{base} \to \pi_{mem}$. 
We assume $T$ reallocates probability mass from the failure set $\mathcal{A}_{avoid}$ (where rewards $r < \epsilon$) to the success set $\mathcal{A}_{suggest}$ (where rewards $r > \epsilon'$), 
such that for any reward threshold $x$, the cumulative mass moved from $[x, \infty)$ to $(-\infty, x)$ is zero.
\end{assumption}

\begin{theorem}[First-Order Stochastic Dominance]
\label{thm:fsd}
If the value estimates $Q^+$ and $Q^-$ are $\epsilon$-consistent with the true reward landscape, the memory-guided policy $\pi_{mem}$ yields a reward distribution that first-order stochastically dominates the base policy: $X_{mem} \succeq_1 X_{base}$.
\end{theorem}

\begin{theorem}[Efficiency of Inference Scaling]
\label{thm:efficiency}
For any sample budget $N \ge 1$, the expected maximum reward satisfies $J_{mem}(N) \ge J_{base}(N)$. 
Furthermore, to reach a target reward $x^*$ with confidence $1-\delta$, $\pi_{mem}$ requires a sample size $N_{mem} \le N_{base}$.
\end{theorem}
We further provide the detailed proof in Appendix~\ref{sec:app:fsd} and ~\ref{sec:app:gainbon}.

\section{Experiments}
\stitle{Benchmarks}. To demonstrate the effectiveness of {\ours}, we follow AgentBoard~\cite{chang2024agentboard} to evaluate our method on a variety of domains, including AlfWorld~\cite{shridhar2020alfworld}, ScienceWorld~(SciWorld)~\cite{wang2022scienceworld}, PDDL~\cite{vallati20152014}, and Tool-Query~(Tool)~\cite{chang2024agentboard}.
The details of each benchmark are available in Appendix~\ref{sec:app:ben}.

\stitle{Baselines}. In this paper, we consider three types of baselines: memory-based, self-evolving, and inference-scaling agents. The former includes  Voyager~\cite{wang2024voyager},  MemoryBank~\cite{zhong2024memorybank} and Generative Agents~\cite{generative-agents-simulacra}. The self-evolving agents include A-Mem~\cite{xu2025amem},  ReasoningBank~\cite{ouyang2025reasoningbank}, and AWM~\cite{wang2025agent}.  The latter contains Vanilla Best-of-N~\cite{stiennon2020learning}, Linear Search~(Reflexion)~\cite{shinn2023reflexion}, and Scattered Forest Search~(SFS)~\cite{light2025sfs}. The details of each baseline and setup are provided in Appendix~\ref{sec:app:base}.

\stitle{Setups \& Metrics}. To ensure a fair comparison, we adopt the Best-of-N setting for our experiments, where the highest-scoring trajectory is selected as the final result for each task across all methods. For memory-based approaches, historical trajectories are recorded and utilized as experience. We adapted the search-based baselines to suit agentic scenarios; specifically, for SFS, we utilized historical trajectories as scattering samples. 
To validate the effectiveness of our method, we evaluate it using three open-source backend LLMs and one closed source LLM: \textit{Qwen2.5-7B~(Qwen)}~\cite{yang2025qwen3}, \textit{Llama3.1-8B~(Llama)}~\cite{dubey2024llama}, \textit{Gemma3-12B~(Gemma)}~\cite{team2025gemma} and \textit{GPT-4o-mini~(GPT)}~\cite{openai_gpt4o_mini}. Across all benchmarks, we set the maximum step count to 10 and the inference budget to $N=64$. In our specific setup, the first 32 trajectories serve as warm-up knowledge, while our method is applied to the remaining 32 trajectories. We set $K=3$ as the default for actions to suggest and avoid.
Following AgentBoard~\cite{chang2024agentboard}, we report both \textit{Progress Rate(PR)} and \textit{Success Rate(SR)}.
\begin{table*}[t]
\caption{The average of four LLMs' inference scaling results, Success Rate(SR \%)/Progress Rate(PR \%), across benchmarks. The average is the mean of four benchmarks. The best performances are in \textcolor{cyan}{\textbf{bold}}, and the second-best method is \textcolor{orange}{\underline{underlined}}.}
\label{tab:main_results}
\centering
\scalebox{0.80}{ 
\begin{tabular}{lccccc} 
   \Xhline{1.2pt} 
   \textbf{Method} & \textbf{\alfw} & \textbf{\sciw} & \textbf{\pddl} & \textbf{\tool} & \textbf{Average}  \\ %
    \Xhline{1.2pt} 
    
    \rowcolor{gray!20} \multicolumn{6}{c}{\textbf{Inference Scaling Methods}} \\
     Vanilla & 26.31	/	58.52	
     &	21.67	/	\textcolor{orange}{\underline{54.88}}	&	5.83	/	\textcolor{orange}{\underline{29.74}}	&	50.83	/	\textcolor{orange}{\underline{82.96}}	&	26.16	/	\textcolor{orange}{\underline{56.53}}(0.00\%	/  \textcolor{orange}{\underline{0.00\%}}	)\\
     SFS & 21.64	/	53.82	
     &	19.17	/	52.04	&	5.42	/	28.99	&	\textcolor{orange}{\underline{51.67}}	/	82.92	&	24.47	/	54.44	($\downarrow$-6.44\%	/  $\downarrow$-3.68\%	)\\
     Reflexion & 18.66	/	58.18	
     &	19.17	/	48.65	&	1.67	/	14.26	&	43.75	/	79.11	&20.81	/	50.05	($\downarrow$-20.45\%	/ $\downarrow$-11.46\%	)\\

    \cmidrule{1-6}
    \rowcolor{gray!20} \multicolumn{6}{c}{\textbf{Memory Mechanism Methods}} \\
     Voyager & \textcolor{orange}{\underline{36.75}}	/	65.47	
     &	8.89	/	36.89	&	5.83	/	25.25	&	39.58	/	76.92	&	
     22.76	/	51.13	($\downarrow$-12.98\%/$\downarrow$-9.54\%)\\
     Generative & 36.01	/	63.33	
     &	11.11	/	37.89	&	\textcolor{orange}{\underline{6.25}}	/	27.41	&	36.67	/	76.41	&22.51	/	51.26	($\downarrow$-13.96\%	/$\downarrow$-9.32\%	)\\
     MemoryBank & 31.16	/	61.54	
     &	12.22	/	38.11	&	5.42	/	26.21	&	40.42	/	77.78	&	22.30	/	50.91	($\downarrow$-14.74\%	/ $\downarrow$-9.94\%	)\\

    \cmidrule{1-6}
    \rowcolor{gray!20} \multicolumn{6}{c}{\textbf{Self-Evolving Methods}} \\ 
     AWM & 27.80	/	59.78	
     &	19.73	/	52.78	&	5.42	/	29.25	&	50.42	/	82.69	&	25.84	/	56.13	($\downarrow$-1.23\%	/ $\downarrow$-0.71\%	) \\
     A-Mem & 36.01	/	\textcolor{cyan}{\textbf{66.39}}	
     &	\textcolor{orange}{\underline{24.45}}	/	49.35	&	\textcolor{orange}{\underline{6.25}}	/	27.53	&	49.58	/	81.26	&	\textcolor{orange}{\underline{29.07}}	/	56.13	($\uparrow$\textcolor{orange}{\underline{11.13\%}}	/ $\downarrow$-0.70\%	)\\
     ReasoningBank & 4.49	/	27.97	
     &	5.00	/	25.11	&	3.33	/	16.31	&	37.50	/	66.81	&	12.58	/	34.05	($\downarrow$-51.91\%	/ $\downarrow$-39.76\%	)\\
    
    \cmidrule{1-6}
     {\oursn}(Ours)  &  \textcolor{cyan}{\textbf{40.30}}	/	 \textcolor{orange}{\underline{65.84}}	
     &	 \textcolor{cyan}{\textbf{26.95}}	/	 \textcolor{cyan}{\textbf{59.63}}	&	 \textcolor{cyan}{\textbf{7.08}}	/	 \textcolor{cyan}{\textbf{30.87}}	&	 \textcolor{cyan}{\textbf{52.09}}	/	 \textcolor{cyan}{\textbf{83.72}}	&	 \textcolor{cyan}{\textbf{31.60}}	/	\textcolor{cyan}{\textbf{60.01}}	($\uparrow$\textcolor{cyan}{\textbf{20.81\%}}	/ $\uparrow$\textcolor{cyan}{\textbf{6.17\%}}	)

\\
    \Xhline{1.2pt}
\end{tabular}
}
\end{table*}

\subsection{Main Results}
\stitle{Endpoint Analysis of Inference Scaling}. To evaluate the effectiveness of our method, we report the average best-of-$N$ results across four LLMs for each benchmark in Table~\ref{tab:main_results}. Detailed results for individual LLMs are provided in Appendix~\ref{sec:app:detailed_exp}. As shown in the table, {\ours} consistently outperforms all baselines, improving the success rate by $20.81\%$ and the progress rate by $6.17\%$ over the vanilla approach. Our method achieves the largest number of best and second-best results across benchmarks. Moreover, {\ours} yields consistent performance gains over the vanilla baseline on most benchmarks. Notably, on AlfWorld, our method achieves a substantial improvement of $53.17\%$ compared to the stateless vanilla baseline. These results demonstrate both the effectiveness and generalization capability of our framework.

Compared to the second-best method, A-Mem, {\ours} achieves an average improvement of approximately $5\%$ in progress rate, along with a slight advantage in success rate. In addition, our framework exhibits more robust performance across all benchmarks. For instance, on the Tool benchmark, {\ours} continues to improve performance, whereas A-Mem fails to show gains. We attribute these improvements to the fine-grained action suggestions enabled by our graph-based modeling, which provides more structured and informative guidance during inference. We further report LLM analysis tables in Appendix~\ref{sec:app:detailed_exp}, which indicate that our framework maintains strong and stable performance across different backbone LLMs.

\begin{figure}[t]
\centering
  \begin{subfigure}[h] {0.49\textwidth}
        \includegraphics[width=0.96\textwidth]{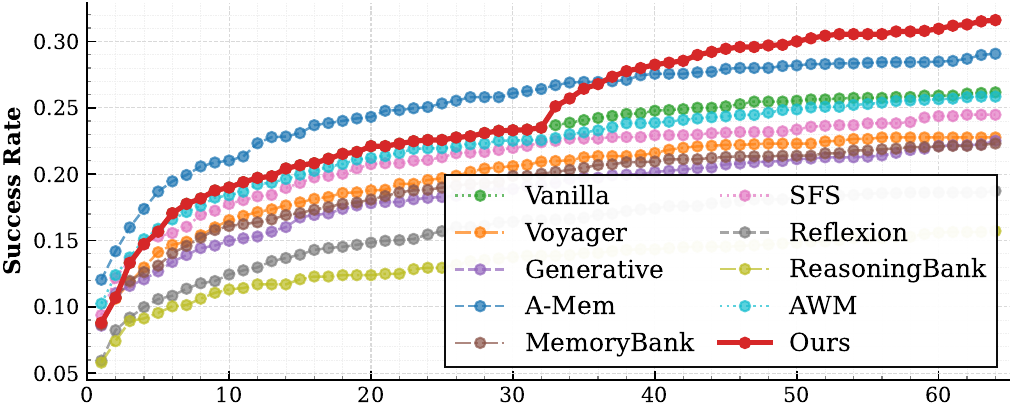}
        \caption{Success Rate}
        \label{fig:success_ratio_avg}
  \end{subfigure}
  \begin{subfigure}[h]{0.49\textwidth}
        \includegraphics[width=0.96\textwidth]{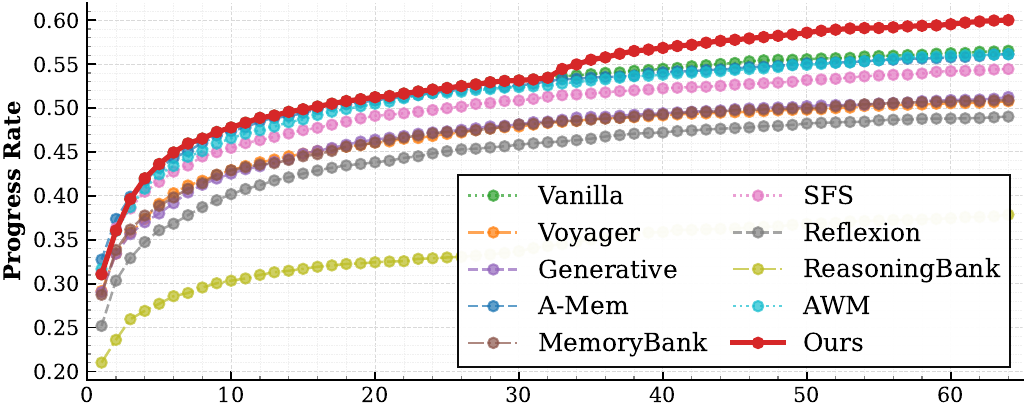}
        \caption{Progress Rate}
        \label{fig:progress_ratio_avg}
  \end{subfigure}
\vspace{-0.2cm}
\caption{Overall inference scaling results across four benchmarks and four LLMs.}
\label{fig:inference_scaling_avg}
\vspace{-0.6cm}
\end{figure}
\stitle{Analysis of Scaling Behavior}. While Table~\ref{tab:main_results} reports the performance of {\ours} and the baselines at $N=64$, a comprehensive understanding of our method requires more than an endpoint comparison.
In this section, we analyze the scaling behavior of inference-time computation, examining how performance improvements correlate with an increasing number of inference steps. 
Figure~\ref{fig:inference_scaling_avg} presents a detailed scaling analysis by illustrating performance trends across the full range of $N$. We further 
visualize the evolution of success rate and progress rate as inference iterations increase.
\begin{wraptable}{r}{0.60\textwidth}
\caption{The average token(M) consumption of Qwen and Llama in one inference. 
}
\label{tab:token_results}
\centering
\scalebox{0.86}{ 
\begin{tabular}{l@{\hspace{0.01cm}}c@{\hspace{0.15cm}}c@{\hspace{0.15cm}}c@{\hspace{0.22cm}}c@{\hspace{0.1cm}}c} 
   \Xhline{1.2pt} 
   \textbf{Method} & \textbf{\alfw}  & \textbf{\sciw} & \textbf{\pddl} & \textbf{\tool} & \textbf{Average}  \\ %
    \Xhline{1.2pt} 
    
    \rowcolor{gray!20} \multicolumn{6}{c}{\textbf{Inference Scaling Methods}} \\
     Vanilla   &   1.2387		&	1.3276	&	1.1399	&	0.9573	&	1.1659 \\
     SFS       &   1.1603		&	1.3259	&	1.1495	&	0.7404	&	1.0940 \\
     Reflexion &   1.9549		&	1.7999	&	1.5172	&	1.1684	&	1.6101 \\

    \cmidrule{1-6}
    \rowcolor{gray!20} \multicolumn{6}{c}{\textbf{Memory Mechanism Methods}} \\
     Voyager    & 1.6485		    &	1.7440	&	1.5582	&	1.2718	&	1.5556 \\
     Generative &  1.8635		&	1.8513	&	1.5255	&	1.1791	&	1.6048 \\
     MemoryBank &  1.8775		&	1.8465	&	1.5904	&	1.2687	&	1.6458 \\

    \cmidrule{1-6}
    \rowcolor{gray!20} \multicolumn{6}{c}{\textbf{Self-Evolving Methods}} \\ 
     AWM           &  1.2403		&	1.3339	&	1.1354	&	0.9102	&	1.1549 \\
     A-Mem         &  3.8703		&	2.8706	&	2.5977	&	1.7104	&	2.7623 \\
     ReasoningBank &   7.6247		&	6.4736	&	3.3232	&	3.6072	&	5.2572 \\
    
    \cmidrule{1-6}
     {\oursn}(Ours)  & 1.3300	&	2.0618	&	1.2403	&	1.0267	&	1.4147 \\
    \Xhline{1.2pt}
\end{tabular}
}
\end{wraptable}

As shown in the figures, our method consistently achieves higher returns than all baselines. After incorporating warm-up knowledge, the performance curve of {\ours} becomes the best across the entire scaling regime. The gains from increased inference-time scaling surpass those of competing methods. Notably, even at $N=64$, {\ours} maintains the steepest performance slope, indicating that further increases in $N$ would continue to widen the performance gap between our method and the baselines.

\stitle{Token Consumption Analysis}. To examine whether our method provides advantages in token efficiency, we report the token consumption of Qwen and Llama across all benchmarks. The results are summarized in Table~\ref{tab:token_results}.  {\ours} incurs higher token costs only compared to SFS, AWM, and the vanilla baseline, while remaining more efficient than other memory-based methods. Compared with the A-Mem, our method saves $50\%$ tokens on average.

\begin{wrapfigure}{r}{0.35\textwidth}
\vspace{-0.3cm}
\centering
\includegraphics[width=0.35\textwidth]{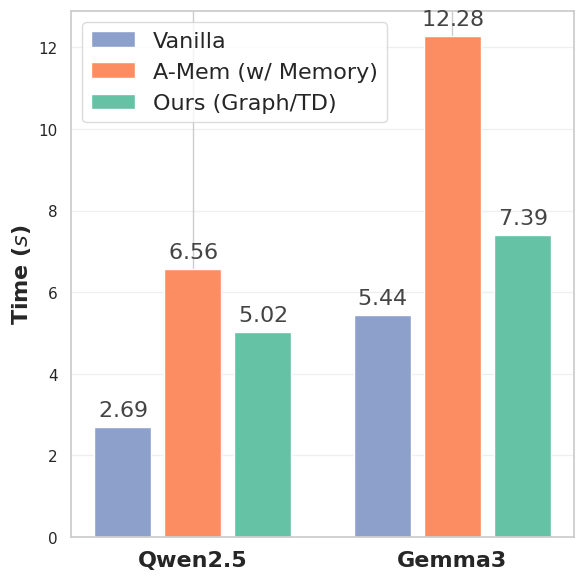} 
\caption{Average inference time of a task on the benchmark ScienceWorld with two open-source LLMs.}
\label{fig:run_avg}
\vspace{-0.4cm}
\end{wrapfigure}
SFS achieves lower token consumption than the vanilla baseline because, in our implementation, trajectory samples are used to replace the original examples. When the sampled trajectories contain fewer tokens than the original examples, SFS naturally incurs a lower token cost. For AWM, the reduced token usage is primarily due to its low success rate, which results in fewer stored workflow memories. In contrast, our method does not require external LLM calls for memory summarization, making the number of additional tokens highly predictable. This leads to more stable and controllable token consumption compared to approaches such as A-Mem and ReasoningBank.

\stitle{Inference Time Comparison}. To demonstrate the efficiency of our framework, we report the time consumed on the benchmark SciWorld. In our experiments, the average time for graph construction in a single SciWorld task is 0.0013 seconds, and the average TD learning time per task is less than 4 seconds, indicating low computational overhead. These operations are based on a single thread of Intel(R) Xeon(R) Gold 6258R CPU. We also provide an inference time comparison on the SciWorld benchmark between our method and the second-best baseline, A-Mem. As the results show in Figure~\ref{fig:run_avg}, our method demonstrates a clear advantage in terms of time efficiency. This improvement stems from the fact that our approach does not rely on additional large language model (LLM) calls for memory storage and retrieval, which are computationally expensive.

\begin{figure}[t]
\centering
\includegraphics[width=1.0\textwidth]{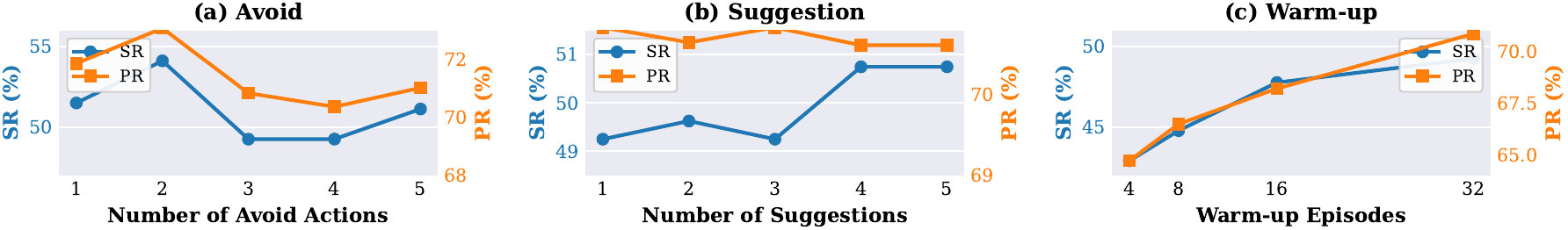} 
\caption{Hyperparameter analysis on AlfWorld with Qwen and Llama.}
\label{fig:ablation_avg}
\vspace{-0.8cm}
\end{figure}

\subsection{Ablation Study}
To provide the insight of {\ours}, we conduct ablation studies, including hyperparameters of actions to suggest and avoid, the variance of the warm-up knowledge, and the effectiveness of components. 

\begin{wraptable}{r}{0.6\textwidth}
\caption{The ablation study results(\%) of components on Tool with Qwen, and SciWorld with Llama and Gemma. S, A and T mean Suggested Action, Action to Avoid, and Trajectory. The best performances are in \textcolor{cyan}{\textbf{bold}}, and the second-best method is \textcolor{orange}{\underline{underlined}}.}
\renewcommand\tabcolsep{3.5pt}
\label{tab:ablation_component}
\centering
\scalebox{1.0}{ 
\begin{tabular}{lcccccc} 
   \Xhline{1.2pt} 
    & \multicolumn{2}{c}{\textbf{\tool}} & \multicolumn{2}{c}{\textbf{\sciw}} & \multicolumn{2}{c}{\textbf{\sciw}}   \\  
    & \multicolumn{2}{c}{\textbf{Qwen}} & \multicolumn{2}{c}{\textbf{Llama}} & \multicolumn{2}{c}{\textbf{Gemma}}   \\  
  \cmidrule(r){2-3}  \cmidrule(r){4-5}  \cmidrule(r){6-7}
   \textbf{Module} & SR & PR & SR & PR & SR & PR  \\
    \midrule
     \rowcolor{gray!20}Vanilla   &   53.33	&	84.31	&	14.44	&	48.99	&	26.67	&	60.89 \\
      +A &   \textcolor{orange}{\underline{60.00}}	&	86.17	&	18.89	&	50.28	&	32.22	&	65.67 \\
     \rowcolor{gray!20} +S &   \textcolor{cyan}{\textbf{61.67}}	&	\textcolor{orange}{\underline{86.74}}	&	16.67	&	51.10	&	\textcolor{cyan}{\textbf{37.78}}	&	69.01 \\
      +T       &  40.00	&	75.72	&	14.44	&	45.73	&	\textcolor{orange}{\underline{34.44}}	&	64.28 \\
     \rowcolor{gray!20}+A+T       &   \textcolor{orange}{\underline{60.00}}	&	85.82	&	20.00	&	50.24	&	36.67	&	\textcolor{cyan}{\textbf{70.56}} \\
      +S+T       &   \textcolor{cyan}{\textbf{61.67}}	&	85.96	&	21.11	&	51.77	&	32.22	&	66.31
 \\
     \rowcolor{gray!20}+A+S       &   \textcolor{orange}{\underline{60.00}}	&	86.32	&	\textcolor{orange}{\underline{22.22}}	&	\textcolor{orange}{\underline{52.36}}	&	\textcolor{orange}{\underline{34.44}}	&	68.40
 \\
      +A+S+T       &   \textcolor{cyan}{\textbf{61.67}}	&	\textcolor{cyan}{\textbf{86.90}}	&	\textcolor{cyan}{\textbf{24.44}}	&	\textcolor{cyan}{\textbf{52.76}}	&	\textcolor{cyan}{\textbf{37.78}}	&	\textcolor{orange}{\underline{69.34 }}
 \\
    \bottomrule
\end{tabular}
}
\end{wraptable}

\stitle{Analysis of Components}. We conduct the component analysis on Tool with Qwen, and SciWorld with Llama and Gemma. To evaluate the contribution of each component, we report the results of an ablation study in Table~\ref{tab:ablation_component}. 
As shown in the table, the suggested actions and actions to avoid components individually contribute to performance improvements.  Conversely, introducing only the historical trajectories to the vanilla baseline yields mixed isolated results, notably decreasing the success rate on the Tool-Query benchmark and showing no initial improvement on SciWorld with Llama. 
However, the historical trajectories module proves crucial when integrated synergistically with the other components, leading to further overall gains. 
For example, on SciWorld with Llama, incorporating both suggested actions and actions to avoid alongside historical trajectories achieves better performance than relying on partial combinations. 
Based on these findings, we adopt all three modules in our final framework to ensure a robust architecture that generalizes well across different benchmarks and backbone LLMs.

\stitle{Analysis of Hyperparameters}. In this part, we conduct the ablation study on the AlfWorld benchmark with Qwen and Llama for hyperparameters. We analyze the sensitivity of our framework to key hyperparameters in Figure~\ref{fig:ablation_avg}. As shown in the figure, our method is robust to variations in both the number of suggested actions and the number of actions to avoid. Across a wide range of values, the performance remains high with low variance, indicating stable behavior under different configurations. Regarding warm-up knowledge, we observe that using a larger number of warm-up trajectories leads to better performance. Detailed hyperparameter analyses for all benchmarks are provided in Appendix~\ref{sec:app:detailed_ablation}.

\section{Conclusion}
In this paper, we propose {\ours}, a novel framework for enhancing inference-time scaling in LLM-based agents. Our approach integrates historical trajectories into an action-centric graph and applies dual-stream iteration with temporal-difference learning to generate structured action guidance for agent decision-making. Compared to traditional memory-based mechanisms, {\ours} offers a more efficient and flexible way to leverage trajectory information, as it does not rely on the LLMs' internal reasoning capability and incurs only a small additional token cost during inference. Extensive experiments across four benchmarks and four backbone LLMs demonstrate the effectiveness, robustness, and generalization ability of our framework.

\stitle{Limitations}.
Despite strong empirical performance, our method depends on sufficient warm-up trajectories as initial knowledge. When the warm-up set is too small, the agent may converge to suboptimal behaviors and become trapped in erroneous patterns, degrading performance below the vanilla baseline. Addressing this limitation is left for future work.

\stitle{Broader Impacts}. 
This paper presents a novel framework to enhance inference-time scaling in LLM-based agents, improving both capability and efficiency. Achieving stronger performance with lower token cost, it makes agent systems more practical for real-world deployment.

\bibliography{example_paper}
\bibliographystyle{plain}

\appendix
\newpage
\section{Related Work}
\stitle{Memory Mechanisms in LLM Agents.} Memory is a core component that enables LLM-based agents to accumulate experience and support long-horizon task execution~\cite{wang2024survey,xi2025rise,gao2024large,li2024survey-mas}. 
Early works primarily addressed context window limitations through inside-trial memory, storing interaction histories via retrieval-augmented generation with dense retrieval~\cite{zhong2024memorybank} or read–write memory structures~\cite{modarressi2024retllm}, as exemplified by MemoryBank~\cite{zhong2024memorybank}, ChatDB~\cite{hu2023chatdb}, MemoChat~\cite{lu2023memochat}, and AgentLite~\cite{liu2024agentlite}. 
Subsequent approaches explored more sophisticated memory management, including cache-like prioritization in MemGPT~\cite{packer2023memgpt} and controller-based memory streams in Self-Controlled Memory (SCM)~\cite{wang2023enhancing}, as well as cross-trial memory extensions and structured abstraction of experiences~\cite{zhao2024expel,zheng2023synapse,generative-agents-simulacra,chhikara2025mem0,salama2025meminsight}, such as A-Mem~\cite{xu2025amem}, which organizes memories via adaptive abstraction, and G-MEM~\cite{zhang2025gmemory}, which models memory as structured graphs to capture relational knowledge.  ReasoningBank~\cite{ouyang2025reasoningbank} introduces a self-evolving framework learning knowledge from both success and failure trajectories and AWM~\cite{wang2025agent} focuses on task-oriented efficiency by encoding successful execution patterns into a dedicated workflow memory.
Despite these advances, most existing methods rely on rigid memory schemas and fixed write–retrieve workflows, limiting their adaptability and generalization across diverse real-world tasks. Designing a flexible and universal memory system that supports robust long-term interactions, therefore, remains an open challenge.

\stitle{Test-time Inference Scaling.}
Recent research has formalized test-time scaling as a paradigm to enhance LLM performance by allocating additional computational resources during inference~\cite{brown2024large,wu2025inference,madaan2023self,shinn2023reflexion,wang2023selfconsistency}, rather than scaling model parameters. The topological organization of subproblems categorizes these methods, determining how a complex query is decomposed and solved. The most prominent sequential approach is CoT, which encourages models to generate intermediate reasoning steps~\cite{Wei_COT}. This method effectively extends the computation depth linearly, allowing the model to decompose problems into a sequence of logical deductions. To address the limitations of linear reasoning, non-sequential structures~\cite{saha-etal-2024-branch,zhang2024rest,qi2025mutual} such as ToT~\cite{yao2023tree} and Graph of Thoughts (GoT)~\cite{besta2024graph} have been introduced. These frameworks utilize tree or graph search algorithms to explore multiple reasoning paths~\cite{xie2023self,liang2025mcts}, enable backtracking, and aggregate information from distinct thought branches~\cite{he2024enhancing,saha-etal-2024-branch}. By treating inference as a search problem over a space of thoughts, these methods trade increased test-time compute for superior problem-solving accuracy, particularly in tasks requiring planning and lookahead.

\section{Proof of Theoretical Analysis}
\label{sec:app:pta}

\subsection{Proof of Theorem \ref{thm:fsd}}
\label{sec:app:fsd}
\begin{proof}
Let $p_{base}(x)$ and $p_{mem}(x)$ be the probability density functions (PDF) of $X_{base}$ and $X_{mem}$ respectively. 
The memory guidance mechanism modifies the base distribution by suppressing low-reward trajectories and enhancing high-reward ones.
Define the probability shift function $\Delta p(x) = p_{mem}(x) - p_{base}(x)$. 
By the law of total probability, $\int_{-\infty}^{\infty} \Delta p(x) dx = 0$.

According to the Dual-Stream strategy, we define a pivot reward threshold $\epsilon$. 
The transformation ensures:
\begin{itemize}[itemsep=1.5pt,topsep=0pt,parsep=0pt] 
    \item For $x < \epsilon$ (Failure regions), $\Delta p(x) \le 0$.
    \item For $x > \epsilon$ (Success regions), $\Delta p(x) \ge 0$.
\end{itemize}
The difference in CDF is given by:
$$F_{mem}(x) - F_{base}(x) = \int_{-\infty}^{x} \Delta p(t) dt $$

\textbf{Case 1}: $x < \epsilon$. Since $\Delta p(t) \le 0$ for all $t \le x < \epsilon$, the integral $\int_{-\infty}^{x} \Delta p(t) dt \le 0$. 
Thus, $F_{mem}(x) \le F_{base}(x)$.

\textbf{Case 2}: $x \ge \epsilon$. We know that $\int_{-\infty}^{\infty} \Delta p(t) dt = 0$. Therefore:
$$\int_{-\infty}^{x} \Delta p(t) dt = \int_{-\infty}^{\infty} \Delta p(t) dt - \int_{x}^{\infty} \Delta p(t) dt = 0 - \int_{x}^{\infty} \Delta p(t) dt$$
For $t \ge x \ge \epsilon$, we have $\Delta p(t) \ge 0$.  Thus, $\int_{x}^{\infty} \Delta p(t) dt \ge 0$, which implies:
$$F_{mem}(x) - F_{base}(x) = - \int_{x}^{\infty} \Delta p(t) dt \le 0 $$
Combined, $F_{mem}(x) \le F_{base}(x)$ for all $x \in \mathbb{R}$. This confirms $X_{mem} \succeq_1 X_{base}$.
\end{proof}

\subsection{Proof of Theorem \ref{thm:efficiency}}
\label{sec:app:gainbon}
\begin{proof}
\stitle{Expected Maximum Reward}:
Let $X^{(1)}, X^{(2)}, \dots, X^{(N)}$ be $N$ independent and identically distributed (i.i.d.) samples drawn from a distribution $F$. We define the random variable $M_N = \max \{X^{(1)}, X^{(2)}, \dots, X^{(N)}\}$ as the maximum reward obtained in $N$ trials. The cumulative distribution function (CDF) of $M_N$ is given by:
$$P(M_N \le x) = P(X^{(1)} \le x, \dots, X^{(N)} \le x) = \prod_{i=1}^N P(X^{(i)} \le x) = [F(x)]^N $$
The expected value of the maximum reward, $\mathbb{E}[M_N]$, can be expressed using the survival function identity:
\begin{equation}
    \mathbb{E}[M_N] = \int_{0}^{\infty} (1 - [F(x)]^N) dx - \int_{-\infty}^{0} [F(x)]^N dx
\end{equation}
Let $\Delta J = J_{mem}(N) - J_{base}(N)$. Substituting the CDFs:
\begin{equation}
    \Delta J = \int_{-\infty}^{\infty} ([F_{base}(x)]^N - [F_{mem}(x)]^N) dx
\end{equation}
From Theorem \ref{thm:fsd}, we have $F_{mem}(x) \le F_{base}(x)$ for all $x \in \mathbb{R}$. 
Since $F(x) \in [0, 1]$ and the power function $f(t) = t^N$ is monotonically non-decreasing for $t \ge 0$, the inequality is preserved under the $N$-th power:
\begin{equation}
    [F_{mem}(x)]^N \le [F_{base}(x)]^N, \quad \forall x \in \mathbb{R}    
\end{equation}
Thus, $\Delta J \ge 0$, which proves that memory-guided inference yields a higher or equal expected maximum reward for any budget $N$.

\stitle{Sample Complexity}:
We now analyze the number of samples $N$ required to achieve a target reward $x^*$ with a success probability at least $1-\delta$. The success condition is $P(M_N > x^*) \ge 1 - \delta$, which is equivalent to:
\begin{equation}
    1 - [F(x^*)]^N \ge 1 - \delta \implies [F(x^*)]^N \le \delta
\end{equation}
Taking the natural logarithm on both sides (noting that $\ln F(x^*) \le 0$):
\begin{equation}
    N \ln F(x^*) \le \ln \delta \implies N \ge \frac{\ln \delta}{\ln F(x^*)}
\end{equation}
Let $N_{mem}$ and $N_{base}$ be the minimum samples required for each policy. Since $F_{mem}(x^*) \le F_{base}(x^*)$, it follows that $\ln F_{mem}(x^*) \le \ln F_{base}(x^*)$. In terms of absolute values, $|\ln F_{mem}(x^*)| \ge |\ln F_{base}(x^*)|$. Therefore:
\begin{equation}
N_{mem} = \left\lceil \frac{\ln \delta}{\ln F_{mem}(x^*)} \right\rceil \le \left\lceil \frac{\ln \delta}{\ln F_{base}(x^*)} \right\rceil = N_{base}
\end{equation}
This inequality demonstrates that {\ours} reaches the same performance target with a lower or equal sampling budget, effectively accelerating the inference scaling process.

\end{proof}

\section{Experimental Setup}
\label{sec:app:setup}

\subsection{Benchmarks}
\label{sec:app:ben}
In this section, we provide a description of the benchmarks in this subsection. The detailed benchmark information is provided as follows.

\textbf{AlfWorld}~\cite{shridhar2020alfworld} is a text-based embodied environment derived from the ALFRED benchmark, comprising 134 test tasks in which agents must perform long-horizon household tasks via natural-language interactions. Each task requires the agent to execute a sequence of grounded actions (e.g., navigation, object manipulation) under partial observability. AlfWorld poses significant challenges for inference-time reasoning due to sparse rewards, compositional task structure, and strong dependencies between early and late-stage actions, making it well-suited for evaluating trajectory-level reasoning and memory reuse.

\textbf{ScienceWorld}~\cite{wang2022scienceworld} focuses on scientific reasoning and experimentation in a simulated environment with 90 test tasks. Agents interact with virtual laboratories to conduct multi-step experiments, manipulate physical variables, and infer latent scientific rules. Compared to AlfWorld, ScienceWorld places greater emphasis on causal reasoning, hypothesis testing, and long-term dependency tracking, which requires agents to retain and reuse intermediate reasoning states across inference steps.

\textbf{PDDL}~\cite{vallati20152014} tasks are based on 60 test symbolic planning problems defined using the Planning Domain Definition Language. In these environments, agents are required to generate valid action sequences that satisfy preconditions and achieve specified goals under deterministic transition dynamics. PDDL domains emphasize structured state transitions and precise action dependencies, providing a clean testbed for evaluating whether inference-time scaling methods can efficiently explore and reuse promising action trajectories in a combinatorial search space.

\textbf{Tool-Query}~\cite{chang2024agentboard} evaluates an agent’s ability to invoke external tools and APIs to answer user queries. Tasks typically require selecting appropriate tools, composing tool calls, and integrating tool outputs into coherent final answers. 60 test environments stress tool selection and execution planning at inference time, where reusing successful action patterns and pruning low-value tool trajectories is critical for efficient reasoning.

\begin{table*}[!h]
\caption{The inference scaling results across benchmarks. The best performances are in \textcolor{cyan}{\textbf{bold}}, and the second-best method is \textcolor{orange}{\underline{underlined}}.}
\label{tab:detailed_detection}
\centering
\footnotesize 
\scalebox{1.0}{ 
\begin{tabular}{llccccl}
   \Xhline{1.2pt}
    & \textbf{Method} & \textbf{\alfw} & \textbf{\sciw} & \textbf{\pddl} & \textbf{\tool} & \textbf{Average}  \\
    \Xhline{1.2pt}
    \multirow{13}{*}{\rotatebox[origin=c]{90}{\textbf{Qwen2.5}}}
    &  \multicolumn{6}{c}{\cellcolor{gray!20}\textbf{Inference Scaling Methods}} \\
    & Vanilla & 38.81	/	67.35		&	\textcolor{orange}{\underline{21.11}}	/	44.05	&	\textcolor{cyan}{\textbf{8.33}}	/	\textcolor{orange}{\underline{31.94}}	&	53.33	/	84.31	&	30.40	/	56.91  \\
    & SFS & 30.60	/	62.00	&	17.78	/	39.58	&	\textcolor{orange}{\underline{6.67}}	/	30.83	&	55.00	/	83.79	&	27.51	/	54.05 \\
    & Reflexion & 34.33	/	72.01	&	14.44	/	45.09	&	1.67	/	15.98	&	50.00	/	81.41	&	25.11	/	53.62 \\
    \cmidrule{2-7}
    &  \multicolumn{6}{c}{\cellcolor{gray!20}\textbf{Memory Mechanism Methods}} \\
    & Voyager & 41.04	/	68.84	&	5.56	/	28.24	&	\textcolor{cyan}{\textbf{8.33}}	/	28.47	&	43.33	/	77.96	&	24.57	/	50.88  \\
    & Generative & 30.60	/	59.76	&	11.11	/	31.45	&	\textcolor{cyan}{\textbf{8.33}}	/	29.83	&	38.33	/	78.98	&	22.09	/	50.01  \\
    & MemoryBank & 29.10	/	59.89	&	13.33	/	36.88	&	\textcolor{cyan}{\textbf{8.33}}	/	27.78	&	41.67	/	77.72	&	23.11	/	50.57  \\
    \cmidrule{2-7}
    &  \multicolumn{6}{c}{\cellcolor{gray!20}\textbf{Self-Evolving Methods}} \\
    & AWM & 41.79	/	70.40	&	15.56	/	39.07	&	\textcolor{cyan}{\textbf{8.33}}	/	31.11	&	51.67	/	83.61	&	29.34	/	56.05  \\
    & A-Mem & \textcolor{orange}{\underline{54.48}}	/	\textcolor{cyan}{\textbf{81.59}}	&	\textcolor{cyan}{\textbf{27.78}}	/	\textcolor{cyan}{\textbf{48.93}}	&	\textcolor{cyan}{\textbf{8.33}}	/	29.69	&	\textcolor{orange}{\underline{58.33}}	/	\textcolor{orange}{\underline{84.61}}	&	\textcolor{orange}{\underline{37.23}}	/	\textcolor{orange}{\underline{61.21}}  \\
    & ReasoningBank & 8.27	/	40.35	&	0.00	/	8.89	&	0.00	/	8.91	&	0.00	/	20.49	&	2.07	/	19.66  \\
    \cmidrule{2-7}
    & {\oursn}(Ours)  & \textcolor{cyan}{\textbf{69.40}}	/	\textcolor{orange}{\underline{80.04}}	&	18.89	/	\textcolor{orange}{\underline{47.86}}	&	\textcolor{cyan}{\textbf{8.33}}	/	\textcolor{cyan}{\textbf{32.36}}	&	\textcolor{cyan}{\textbf{61.67}}	/	\textcolor{cyan}{\textbf{86.90}}	&	\textcolor{cyan}{\textbf{39.57}}	/	\textcolor{cyan}{\textbf{61.79}}\\
 \hline
 \hline
    \multirow{13}{*}{\rotatebox[origin=c]{90}{\textbf{Llama3.1}}}
    &  \multicolumn{6}{c}{\cellcolor{gray!20}\textbf{Inference Scaling Methods}} \\
    & Vanilla & 14.18	/	54.66	&	14.44	/	\textcolor{orange}{\underline{48.99}}	&	\textcolor{orange}{\underline{5.00}}	/	\textcolor{cyan}{\textbf{28.03}}	&	\textcolor{orange}{\underline{33.33}}	/	78.20	&	16.74	/	52.47 \\
    & SFS & 12.69	/	48.88	&	12.22	/	43.47	&	\textcolor{cyan}{\textbf{6.67}}	/	25.83	&	\textcolor{cyan}{\textbf{36.67}}	/	\textcolor{cyan}{\textbf{78.71}}	&	17.06	/	49.22  \\
    & Reflexion & 3.73	/	45.09	&	1.11	/	15.76	&	1.67	/	10.19	&	10.00	/	68.01	&	4.13	/	34.76  \\
    \cmidrule{2-7}
    &  \multicolumn{6}{c}{\cellcolor{gray!20}\textbf{Memory Mechanism Methods}} \\
    & Voyager & 13.43	/	54.42	&	1.11	/	5.45	&	3.33	/	19.91	&	20.00	/	71.97	&	9.47	/	37.94  \\
    & Generative & 20.90	/	56.97	&	1.11	/	7.50	&	3.33	/	19.53	&	15.00	/	71.58	&	10.09	/	38.90  \\
    & MemoryBank & \textcolor{orange}{\underline{26.12}}	/	\textcolor{orange}{\underline{59.14}}	&	1.11	/	6.71	&	3.33	/	20.00	&	21.67	/	72.98	&	13.06	/	39.71  \\
    \cmidrule{2-7}
    &  \multicolumn{6}{c}{\cellcolor{gray!20}\textbf{Self-Evolving Methods}} \\
    & AWM & 17.16	/	55.97	& \textcolor{orange}{\underline{15.56}}	/	46.53	&	\textcolor{cyan}{\textbf{6.67}}	/	\textcolor{orange}{\underline{27.64}}	&	\textcolor{orange}{\underline{33.33}}	/	\textcolor{orange}{\underline{78.36}}	&	\textcolor{orange}{\underline{18.18}}	/	\textcolor{orange}{\underline{52.13}}  \\
    & A-Mem & 8.96	/	51.68	&	1.11	/	10.08	&	3.33	/	20.44	&	25.00	/	74.00	&	9.60	/	39.05  \\
    & ReasoningBank & 0.00	/	8.02	&	6.67	/	33.03	&	3.33	/	19.11	&	26.67	/	72.21	&	9.17	/	33.09  \\
    \cmidrule{2-7}
    & {\oursn}(Ours)  & \textcolor{cyan}{\textbf{29.10}}	/	\textcolor{cyan}{\textbf{61.63}}	&	\textcolor{cyan}{\textbf{24.44}}	/	\textcolor{cyan}{\textbf{52.76}}	&	\textcolor{orange}{\underline{5.00}}	/	27.26	&	28.33	/	77.69	&	\textcolor{cyan}{\textbf{21.72}}	/	\textcolor{cyan}{\textbf{54.84}}\\
 \hline
 \hline
    \multirow{13}{*}{\rotatebox[origin=c]{90}{\textbf{Gemma3}}}
    &  \multicolumn{6}{c}{\cellcolor{gray!20}\textbf{Inference Scaling Methods}} \\
    & Vanilla & 25.37	/	50.12		&	26.67	/	60.89	&	5.00	/	31.61	&	65.00	/	86.74	&	30.51	/	57.34 \\
    & SFS & 19.40	/	46.39	&	30.00	/	65.80	&	5.00	/	32.50	&	63.33	/	86.36	&	29.43	/	57.76  \\
    & Reflexion & 21.64	/	58.46	&	34.44	/	\textcolor{cyan}{\textbf{69.84}}	&	1.67	/	11.26	&	\textcolor{orange}{\underline{66.67}}	/	87.33	&	31.11	/	56.72  \\
    \cmidrule{2-7}
    &  \multicolumn{6}{c}{\cellcolor{gray!20}\textbf{Memory Mechanism Methods}} \\
    & Voyager & 37.31	/	62.13	&	13.33	/	56.16	&	\textcolor{orange}{\underline{8.33}}	/	31.08	&	60.00	/	84.46	&	29.74	/	58.46  \\
    & Generative & \textcolor{orange}{\underline{43.28}}	/	\textcolor{cyan}{\textbf{64.80}}	&	22.22	/	56.16	&	\textcolor{cyan}{\textbf{10.00}}	/	\textcolor{cyan}{\textbf{35.58}}	&	61.67	/	81.94	&	34.29	/	59.62  \\
    & MemoryBank & 35.07	/	61.44	&	17.78	/	52.13	&	6.67	/	32.02	&	56.67	/	82.76	&	29.05	/	57.09  \\
    \cmidrule{2-7}
    &  \multicolumn{6}{c}{\cellcolor{gray!20}\textbf{Self-Evolving Methods}} \\
    & AWM & 25.37	/	50.62	&	30.00	/	64.74	&	3.33	/	32.52	&	63.33	/	85.89	&	30.51	/	58.44  \\
    & A-Mem & \textcolor{cyan}{\textbf{44.78}}	/	\textcolor{orange}{\underline{63.06}}	&	\textcolor{cyan}{\textbf{38.89}}	/	68.91	&	\textcolor{orange}{\underline{8.33}}	/	32.65	&	60.00	/	84.25	&	\textcolor{cyan}{\textbf{38.00}}	/	\textcolor{orange}{\underline{62.22}}  \\
    & ReasoningBank & 0.00	/	15.73	&	0.00	/	8.34	&	6.67	/	25.65	&	\textcolor{cyan}{\textbf{68.33}}	/	\textcolor{cyan}{\textbf{89.15}}	&	18.75	/	34.72  \\
    \cmidrule{2-7}
    & {\oursn}(Ours)  & 29.85	/	57.09	&	\textcolor{orange}{\underline{37.78}}	/	\textcolor{orange}{\underline{69.34}}	&	\textcolor{cyan}{\textbf{10.00}}	/	\textcolor{orange}{\underline{35.16}}	&	\textcolor{orange}{\underline{66.67}}	/	\textcolor{orange}{\underline{87.57}}	&	\textcolor{orange}{\underline{36.08}}	/	\textcolor{cyan}{\textbf{62.29}}\\
 \hline
 \hline
    \multirow{13}{*}{\rotatebox[origin=c]{90}{\textbf{GPT}}}
    &  \multicolumn{6}{c}{\cellcolor{gray!20}\textbf{Inference Scaling Methods}} \\
    & Vanilla & 26.87	/	61.94	&	24.44	/	65.60	&	\textcolor{cyan}{\textbf{5.00}}	/	27.38	&	51.67	/	82.59	&	27.00	/	59.38  \\
    & SFS & 23.88	/	58.02	&	16.67	/	59.30	&	\textcolor{orange}{\underline{3.33}}	/	26.80	&	51.67	/	82.82	&	23.89	/	56.74  \\
    & Reflexion & 14.93	/	57.15	&	\textcolor{orange}{\underline{26.67}}	/	63.90	&	1.67	/	19.59	&	48.33	/	79.67	&	22.90	/	55.08 \\
    \cmidrule{2-7}
    &  \multicolumn{6}{c}{\cellcolor{gray!20}\textbf{Memory Mechanism Methods}} \\
    & Voyager & \textcolor{cyan}{\textbf{55.22}}	/	\textcolor{cyan}{\textbf{76.49}}	&	15.56	/	57.69	&	\textcolor{orange}{\underline{3.33}}	/	21.53	&	35.00	/	73.28	&	27.28	/	57.25 \\
    & Generative & \textcolor{orange}{\underline{49.25}}	/	\textcolor{orange}{\underline{71.77}}	&	10.00	/	56.46	&	\textcolor{orange}{\underline{3.33}}	/	24.68	&	31.67	/	73.12	&	23.56	/	56.51 \\
    & MemoryBank & 34.33	/	65.67	&	16.67	/	56.71	&	\textcolor{orange}{\underline{3.33}}	/	25.04	&	41.67	/	77.64	&	24.00	/	56.27  \\
    \cmidrule{2-7}
    &  \multicolumn{6}{c}{\cellcolor{gray!20}\textbf{Self-Evolving Methods}} \\
    & AWM & 26.87	/	62.13	&	17.78	/	60.79	&	\textcolor{orange}{\underline{3.33}}	/	25.74	&	\textcolor{orange}{\underline{53.33}}	/	\textcolor{orange}{\underline{82.90}}	&	25.33	/	57.89  \\
    & A-Mem & 35.82	/	69.22	&	\textcolor{cyan}{\textbf{30.00}}	/	\textcolor{cyan}{\textbf{69.47}}	&	\textcolor{cyan}{\textbf{5.00}}	/	\textcolor{orange}{\underline{27.33}}	&	\textcolor{cyan}{\textbf{55.00}}	/	82.17	&	\textcolor{cyan}{\textbf{31.46}}	/	\textcolor{cyan}{\textbf{62.05}}\\
    & ReasoningBank & 9.70	/	47.76	&	13.33	/	50.16	&	\textcolor{orange}{\underline{3.33}}	/	11.58	&	\textcolor{cyan}{\textbf{55.00}}	/	\textcolor{cyan}{\textbf{85.40}}	&	20.34	/	48.73  \\
    \cmidrule{2-7}
    & {\oursn}(Ours)  & 32.84	/	64.61	&	\textcolor{orange}{\underline{26.67}}	/	\textcolor{orange}{\underline{68.56}}	&	\textcolor{cyan}{\textbf{5.00}}	/	\textcolor{cyan}{\textbf{28.68}}	&	51.67	/	82.72	&	\textcolor{orange}{\underline{29.05}}	/	\textcolor{orange}{\underline{61.14}}\\
    \bottomrule
\end{tabular}
}
\end{table*}

\subsection{Baselines}
\label{sec:app:base}
\textbf{Voyager}~\cite{wang2024voyager} is an LLM-powered embodied agent designed for autonomous lifelong learning. It integrates three core mechanisms: an automatic curriculum that generates adaptive goals to maximize discovery, a dynamic skill library that stores executable code for future retrieval and composition, and an iterative prompting mechanism that leverages environmental feedback and self-verification to refine generated programs. We follow the code of G-Memory~\cite{zhang2025gmemory} to implement. 

\textbf{MemoryBank}~\cite{zhong2024memorybank} introduces an explicit long-term memory mechanism to augment LLM so they can recall and leverage past interactions over extended time spans. It structures memory hierarchically by recording and summarizing conversational events and user characteristics, then uses dense semantic embeddings and efficient vector retrieval to fetch relevant memories during new interactions. Our experiments are based on the implementation of G-Memory~\cite{zhang2025gmemory}.

\textbf{Generative Agents}~\cite{generative-agents-simulacra} is a framework that extends LLMs to create autonomous software agents capable of simulating believable human-like behavior over long time horizons by maintaining and reasoning over rich internal memories. Each agent continuously logs experiences in a memory stream expressed in natural language, retrieves relevant memories based on recency and relevance, and periodically performs reflection to synthesize higher-level insights that shape its long-term behavior. We referred to the implementation method of G-Memory~\cite{zhang2025gmemory}.

\textbf{SFS}~\cite{light2025sfs} introduces an inference-time optimization framework that improves LLM inference scaling by formulating code generation as a black-box search over the program space. It balances exploration and exploitation through three mechanisms: scattering to generate diverse improvement directions, foresting to explore multiple search trees from different initial solutions, and scouting to share feedback across branches. We adopt the source code to fit the benchmarks.

\textbf{Reflexion}~\cite{shinn2023reflexion} uses three components: an Actor to generate actions and text, an Evaluator to score those outputs, and a Self-Reflection module to convert feedback into natural language summaries. After each trial in an environment, the agent’s feedback and observations are transformed into reflective text and stored in an episodic memory buffer, which is then used as context in subsequent trials to guide better decision-making. In our implementation, for each action, we will use the same LLM to generate feedback.

\textbf{A-Mem}~\cite{xu2025amem} is a dynamic memory system designed to enhance LLM Agents' ability to manage and utilize long-term historical information. It employs an LLM to generate and organizes detailed memory notes for each interaction that include contextual descriptions, keywords, and tags, and then uses semantic similarity to autonomously link related memories. It allows agents to autonomously construct, link, and evolve their memories over time, enabling better long-term reasoning and consistency in complex tasks. We modify the source code to meet the benchmarks. 

\textbf{Agent Workflow Memory}~\cite{wang2025agent} introduces a method for storing workflows from experience for multi-step tasks, thereby enhancing the ability of LLM agents.  Unlike traditional memory systems that store raw interaction trajectories, AWM focuses on extracting high-level, reusable sub-routines, essentially patterns of successful actions, that can be generalized across different but related tasks. We customize the source code to match the benchmarks.

\textbf{ReasoningBank}~\cite{ouyang2025reasoningbank} is a framework designed to enable AI agents to evolve autonomously by leveraging a structured memory. It allows the agent to extract reasoning skills from the failure and success trajectories. By treating reasoning as a reusable and evolvable asset, the method allows agents to progressively improve their performance and generalizability through a self-evolving loop. We follow the paper to implement a simple version of the sequential scaling agent.

\subsection{Implementation}
In our implementation, we adopt the base agent architecture from AgentBoard~\cite{chang2024agentboard} across all benchmark environments. Each task episode is strictly capped at a maximum of 10 steps. At each step, the LLM is provided with the complete trajectory of past observations and actions to inform decision-making, alongside few-shot demonstrations within the prompt to enhance performance. For the temporal difference (TD) learning phase, we set the discount factor $\gamma=0.95$ and the learning rate $\alpha=0.05$. Finally, the $Q$-values are updated for 100 iterations per task to ensure convergence. The default top k for suggestion and avoid actions are 3.

\subsection{Compute Resources Used}
The experimental setup incorporates open-source LLMs (Qwen2.5-7B, Gemma3-12B, Llama3.1-8B) alongside a commercial LLM (GPT4o-mini). For GPT4o-mini, we perform inference through the OpenAI API using standard Linux CPU environments. To isolate performance measurements and eliminate cross-model interference, the open-source models are evaluated locally in a strictly sequential manner. Each model is loaded independently onto a single NVIDIA A100 GPU (80GB VRAM) to maintain consistent and reproducible runtime metrics. The complete host system features an Intel Xeon Gold 6258R CPU and eight NVIDIA A100 GPUs (80GB each), configured with CUDA 12.8, Python 3.10, and PyTorch 2.6.0 to readily support iterative and parallel workflows.

\section{Detailed Experiments}
\subsection{Main Results}
\label{sec:app:detailed_exp}

\stitle{Endpoint Analysis of Inference Scaling}.  We provide a detailed comparison between {\ours} and baseline methods across multiple benchmarks and LLM backbones in this subsection. The results are summarized in Table~\ref{tab:detailed_detection}. As shown in Table~\ref{tab:detailed_detection}, {\ours} achieves the highest number of best and second-best results across benchmarks, further demonstrating its overall superiority over competing approaches.

Compared to stateless inference-scaling methods, {\ours} consistently outperforms these baselines in nearly all settings, benefiting from the incorporation of historical memory knowledge that enables more informed decision-making. In contrast to memory-based agent frameworks, {\ours} provides fine-grained, action-level guidance rather than relying on high-level summaries or implicit reasoning. As a result, methods that heavily depend on the intrinsic reasoning capabilities of LLMs exhibit a larger performance gap when paired with smaller models, such as Qwen2.5-7B and Llama3.1-8B.

A similar trend is observed for self-evolving baselines. For example, when using Llama3.1-8B, {\ours} achieves a success rate that is over $50\%$ higher than that of A-Mem. However, when switching to a stronger backbone such as Gemma3-12B, A-Mem attains a higher success rate, suggesting that its performance is more sensitive to the underlying LLM’s reasoning ability. In contrast, as shown in Table~\ref{tab:detailed_detection}, {\ours} delivers consistent improvements across all four LLM backbones. This stability highlights the robustness of our framework and indicates that its effectiveness is less dependent on model scale, making it well-suited for a wide range of deployment scenarios.

\stitle{Token Consumption Analysis}. We report detailed token consumption results for Qwen2.5-7B and Llama3.1-8B in Table~\ref{tab:detailed_token}. As shown in the table, {\ours} consistently demonstrates strong advantages in token efficiency over most baseline methods. In particular, compared with A-Mem, the second-best method in Table~\ref{tab:main_results}, {\ours} achieves an average reduction of approximately $50\%$ in total token usage. Notably, on AlfWorld, our method incurs only about one-third of the token cost of A-Mem, highlighting its efficiency in long-horizon tasks.

For ReasoningBank, the substantially higher token consumption arises from its self-refinement mechanism, which repeatedly invokes the LLM at each reasoning step. This leads to multiple LLM calls per task and results in significantly greater token usage than any other method. In contrast, our framework avoids iterative external reasoning calls and instead relies on structured memory and graph-based action guidance, enabling more efficient inference with predictable and controlled token growth. These results indicate that {\ours} offers a favorable trade-off between performance and computational cost, making it more suitable for large-scale or resource-constrained deployment scenarios.

\begin{table*}[t]
\caption{The detailed token(M) costs results across benchmarks. The LLMs are Qwen2.5-7B and Llama3.1-8B.}
\label{tab:detailed_token}
\centering
\footnotesize 
\scalebox{0.95}{ 
\begin{tabular}{lcccccccccc} 
   \Xhline{1.2pt}
   & \multicolumn{2}{c}{\textbf{\alfw}}  & \multicolumn{2}{c}{\textbf{\sciw}} & \multicolumn{2}{c}{\textbf{\pddl}} & \multicolumn{2}{c}{\textbf{\tool}} & \multicolumn{2}{c}{\textbf{Average}}  \\
   \cmidrule(r){2-3} \cmidrule(r){4-5} \cmidrule(r){6-7} \cmidrule(r){8-9} \cmidrule(r){10-11} 
   \textbf{Method} & Qwen & Llama & Qwen & Llama & Qwen & Llama & Qwen & Llama & Qwen & Llama  \\
    \Xhline{1.2pt} 
    \multicolumn{11}{c}{\cellcolor{gray!20}\textbf{Inference Scaling Methods}} \\
     Vanilla &  1.2340	&	1.2434		&	1.3014	&	1.3539	&	1.1479	&	1.1320	&	0.7732	&	1.1414	&	1.1141	&	1.2177 \\
     SFS & 1.1575	&	1.1631		&	1.2931	&	1.3587	&	1.1429	&	1.1560	&	0.7404	&	1.0985	&	1.0835	&	1.2259 \\
     Reflexion &  1.9294	&	1.9803		&	1.7775	&	1.8222	&	1.5425	&	1.4918	&	0.9832	&	1.3536	&	1.5582	&	1.6620 \\    
   \midrule
    \multicolumn{11}{c}{\cellcolor{gray!20}\textbf{Memory Mechanism Methods}} \\
     Voyager & 1.7297	&	1.5673	&	1.7448	&	1.7432	&	1.5481	&	1.5683	&	0.9958	&	1.5478	&	1.5046	&	1.6066 \\
     Generative & 1.9334	&	1.7936		&	1.7665	&	1.9361	&	1.5069	&	1.5442	&	0.9331	&	1.4250	&	1.5350	&	1.6747 \\
     MemoryBank &  1.7824	&	1.9727		&	1.8327	&	1.8604	&	1.6327	&	1.5481	&	1.0435	&	1.4939	&	1.5728	&	1.7188 \\
   \midrule
    \multicolumn{11}{c}{\cellcolor{gray!20}\textbf{Self-Evolving Methods}} \\
     AWM & 1.2424	&	1.2382		&	1.3071	&	1.3607	&	1.1338	&	1.1369	&	0.7184	&	1.1020	&	1.1004	&	1.2095 \\
     A-Mem &  3.9247	&	3.8159	&	2.8388	&	2.9024	&	2.6547	&	2.5407	&	1.4194	&	2.0015	&	2.7094	&	2.8151 \\
     ReasoningBank &  8.3453	&	6.9041	&	6.1859	&	6.7612	&	3.0506	&	3.5958	&	2.7358	&	4.4785	&	5.0794	&	5.4349 \\
    \midrule
     {\oursn}(Ours)  & 1.2525	&	1.4074	&	2.0428	&	2.0808	&	1.2463	&	1.2343	&	0.7456	&	1.3078	&	1.3218	&	1.5076 \\
    \bottomrule
\end{tabular}
}
\end{table*}

\stitle{Analysis of Scaling Behavior}. To further examine the scaling behavior of {\ours}, we present detailed per-benchmark results in Figure~\ref{fig:inference_scaling_alfworld}. As shown in the figure, our method achieves consistent and steady performance improvements across different benchmarks and LLM backbones, demonstrating strong robustness to both task variation and model choice.

We also observe that several baseline methods, such as ReasoningBank, exhibit relatively low initial success and progress rates. When these methods incorporate negative or low-quality historical trajectories as knowledge, the agent can become trapped in suboptimal action sequences and struggle to recover, limiting the effectiveness of further inference scaling. A similar trend is observed for most other baselines: while increasing the number of scaling operations leads to substantial gains in the early stages, the improvements gradually diminish and plateau as scaling continues.

In contrast, {\ours} continues to deliver meaningful performance gains even in the later stages of inference scaling. This sustained improvement indicates that our framework is more effective at leveraging additional inference steps, allowing the agent to refine its decisions over time rather than prematurely converging to suboptimal behaviors. These results highlight the advantage of structured, action-level guidance in enabling long-term benefits from inference-time scaling.

\subsection{Ablation Study}
\label{sec:app:detailed_ablation}
\stitle{Analysis of Hyperparameters}. In this subsection, we present detailed ablation results analyzing the impact of key hyperparameters. As shown in Table~\ref{tab:ablation_parameters_detail}, {\ours} exhibits robust performance with respect to both the number of suggested actions and the number of actions to avoid, for both success rate and progress rate. Across different LLM backbones, the performance variations remain within a narrow range, indicating low sensitivity to these hyperparameter choices.

Regarding warm-up knowledge, we observe consistent trends across multiple benchmarks and LLMs. In particular, increasing the amount of warm-up trajectories generally leads to a high improved performance, highlighting the importance of sufficient initial knowledge for stable inference scaling. These results suggest that our framework is resilient to hyperparameter variations and can be reliably applied across different tasks and model scales without extensive tuning.

\begin{table}[t]
\caption{The hyperparameter analysis results(\%) of SR and PR on components on AlfWorld with Qwen and Llama. H-Para. means hyperparameters. The best performances are in \textcolor{cyan}{\textbf{bold}}, and the second-best method is \textcolor{orange}{\underline{underlined}}.}
\label{tab:ablation_parameters_detail}
\centering
\renewcommand\tabcolsep{2.5pt}
\scalebox{1.0}{ 
\begin{tabular}{l@{\hspace{0.05cm}}cccc} 
   \Xhline{1.2pt} 
    & & \multicolumn{2}{c}{\textbf{\alfw}} &    \\  
  \cmidrule(r){3-4}  
   & \textbf{H-Para.} & Qwen & Llama  & Average\\
    \midrule
    \multirow{5}{*}{\rotatebox[origin=c]{90}{Avoid}}
    & 1 &  \textcolor{orange}{\underline{72.39}}	/	\textcolor{orange}{\underline{82.09}}	&	30.60	/	61.63	&		51.50	/	\textcolor{orange}{\underline{71.86}}
 \\
    & 2 &   \textcolor{cyan}{\textbf{73.88}}	/	\textcolor{cyan}{\textbf{83.02}}	&	\textcolor{cyan}{\textbf{34.33}}	/	\textcolor{cyan}{\textbf{63.12}}	&		\textcolor{cyan}{\textbf{54.11}}	/	\textcolor{cyan}{\textbf{73.07}}

 \\
    & 3 &  69.40	/	80.04	&	29.10	/	61.63	&		49.25	/	70.84

 \\
    & 4 &   67.91	/	78.73	&	30.60	/	62.00	&		49.26	/	70.37

 \\
    & 5 &   69.40	/	80.22	&	\textcolor{orange}{\underline{32.84}}	/	\textcolor{orange}{\underline{61.82}}	&		\textcolor{orange}{\underline{51.12}}	/	71.02

 \\
    \midrule
    \multirow{5}{*}{\rotatebox[origin=c]{90}{Suggestion}}
    & 1 &   \textcolor{cyan}{\textbf{69.40}}	/	\textcolor{cyan}{\textbf{80.04}}	&	29.10	/	\textcolor{cyan}{\textbf{61.63}}	&		49.25	/	\textcolor{cyan}{\textbf{70.84}}

 \\
    & 2 &   \textcolor{cyan}{\textbf{69.40}}	/	\textcolor{cyan}{\textbf{80.04}}	&	\textcolor{orange}{\underline{29.85}}	/	\textcolor{orange}{\underline{61.26}}	&		\textcolor{orange}{\underline{49.63}}	/	\textcolor{orange}{\underline{70.65}}

 \\
    & 3 &  \textcolor{cyan}{\textbf{69.40}}	/	\textcolor{cyan}{\textbf{80.04}}	&	29.10	/	\textcolor{cyan}{\textbf{61.63}}	&		49.25	/	\textcolor{cyan}{\textbf{70.84}}

 \\
    & 4 & \textcolor{cyan}{\textbf{69.40}}	/	\textcolor{cyan}{\textbf{80.04}}	&	\textcolor{cyan}{\textbf{32.09}}	/	61.19	&		\textcolor{cyan}{\textbf{50.75}}	/	70.62

\\
    & 5 &  \textcolor{cyan}{\textbf{69.40}}	/	\textcolor{cyan}{\textbf{80.04}}	&	\textcolor{cyan}{\textbf{32.09}}	/	61.19	&		\textcolor{cyan}{\textbf{50.75}}	/	70.62

 \\
    \midrule
    \multirow{4}{*}{\rotatebox[origin=c]{90}{Warm-up}}
    & 32 &   \textcolor{cyan}{\textbf{69.40}}	/	\textcolor{cyan}{\textbf{80.04}}	&	\textcolor{cyan}{\textbf{29.10}}	/	\textcolor{cyan}{\textbf{61.63}}	&			\textcolor{cyan}{\textbf{49.25}}	/	\textcolor{cyan}{\textbf{70.84}}

 \\
    & 16 &   \textcolor{cyan}{\textbf{69.40}}	/	\textcolor{orange}{\underline{79.85}}	&	\textcolor{orange}{\underline{26.12}}	/	\textcolor{orange}{\underline{56.53}}	&		\textcolor{orange}{\underline{47.76}}	/	\textcolor{orange}{\underline{68.19}}

 \\
    & 8 &  \textcolor{orange}{\underline{64.93}}	/	77.61	&	24.63	/	55.35	&	44.78	/	66.48

 \\
    & 4 &   \textcolor{orange}{\underline{64.93}}	/	76.68	&	20.90	/	52.74	&	42.92	/	64.71

 \\
    \bottomrule
\end{tabular}
}
\end{table}

\begin{figure*}[h]
\centering
  \begin{subfigure}[h] {0.48\textwidth}
        \includegraphics[width=0.99\textwidth]{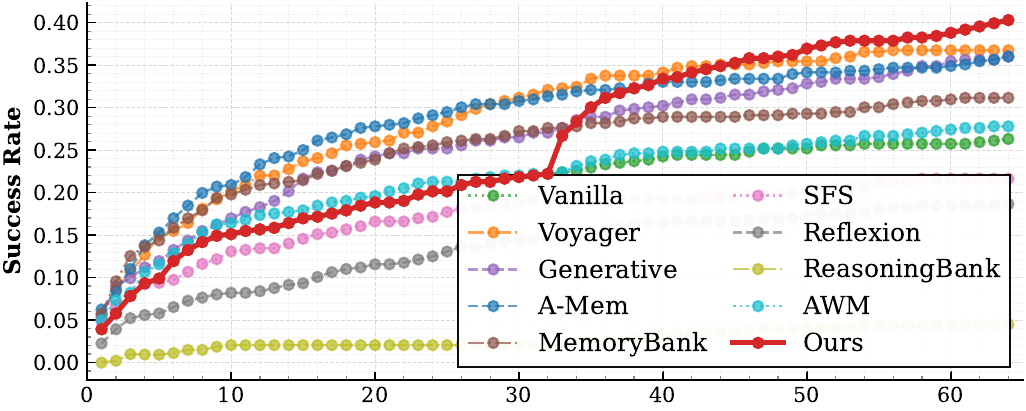}
        \caption{Success Rate of AlfWorld.}
        \label{fig:success_ratio_avg_alf}
  \end{subfigure}
  \begin{subfigure}[h]{0.48\textwidth}
        \includegraphics[width=0.99\textwidth]{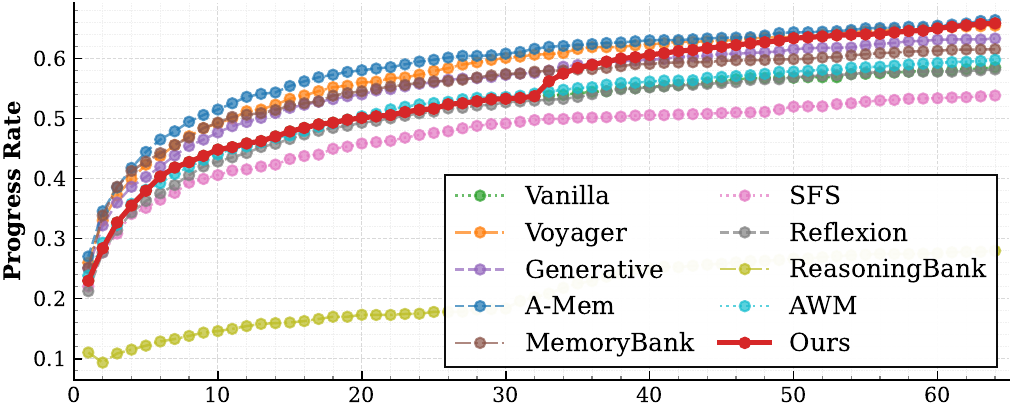}
        \caption{Progress Rate of AlfWorld.}
        \label{fig:progress_ratio_avg_alf}
  \end{subfigure}

  \begin{subfigure}[h] {0.48\textwidth}
        \includegraphics[width=0.99\textwidth]{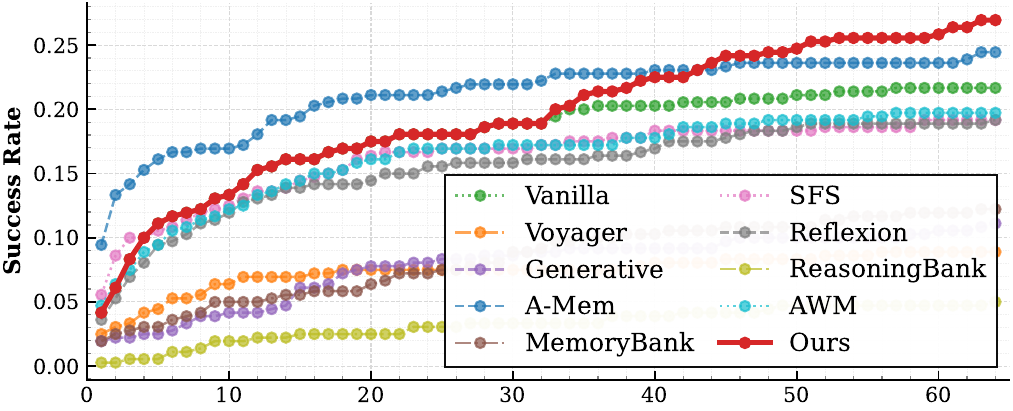}
        \caption{Success Rate of ScienceWorld.}
        \label{fig:success_ratio_avg_sci}
  \end{subfigure}
  \begin{subfigure}[h]{0.48\textwidth}
        \includegraphics[width=0.99\textwidth]{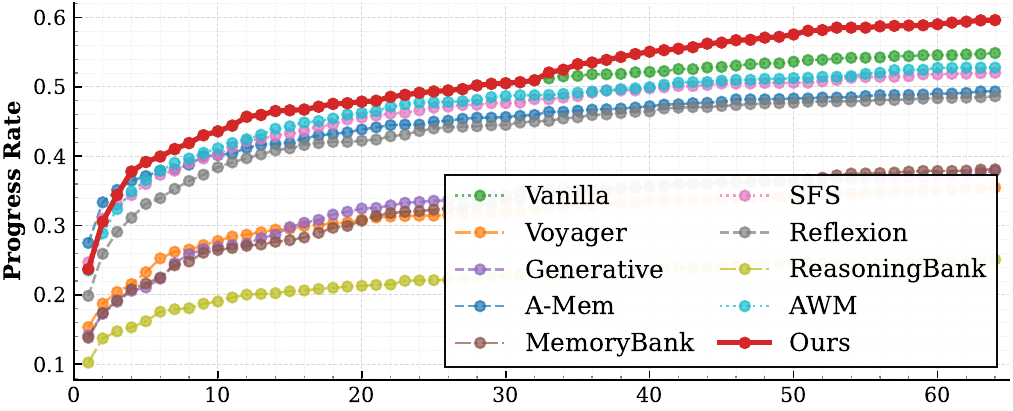}
        \caption{Progress Rate of ScienceWorld.}
        \label{fig:progress_ratio_avg_sci}
  \end{subfigure}
  
  \begin{subfigure}[h] {0.48\textwidth}
        \includegraphics[width=0.99\textwidth]{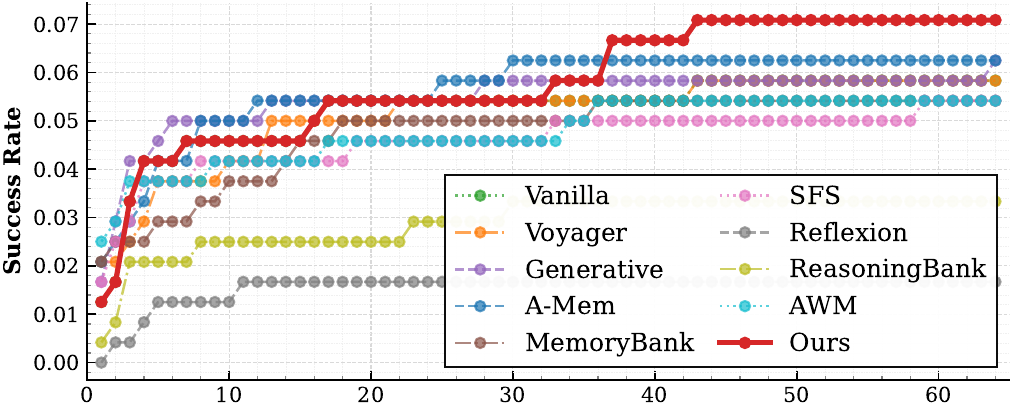}
        \caption{Success Rate of PDDL.}
        \label{fig:success_ratio_avg_pddl}
  \end{subfigure}
  \begin{subfigure}[h]{0.48\textwidth}
        \includegraphics[width=0.99\textwidth]{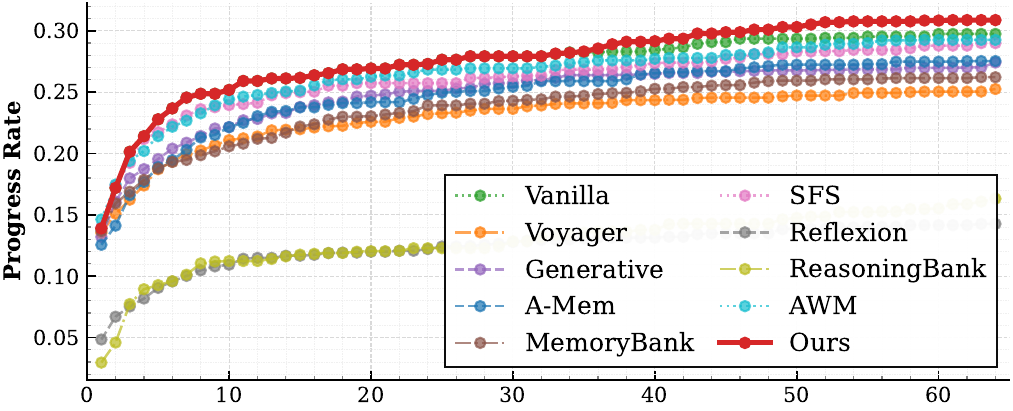}
        \caption{Progress Rate of PDDL.}
        \label{fig:progress_ratio_avg_pddl}
  \end{subfigure}  
  
  \begin{subfigure}[h] {0.48\textwidth}
        \includegraphics[width=0.99\textwidth]{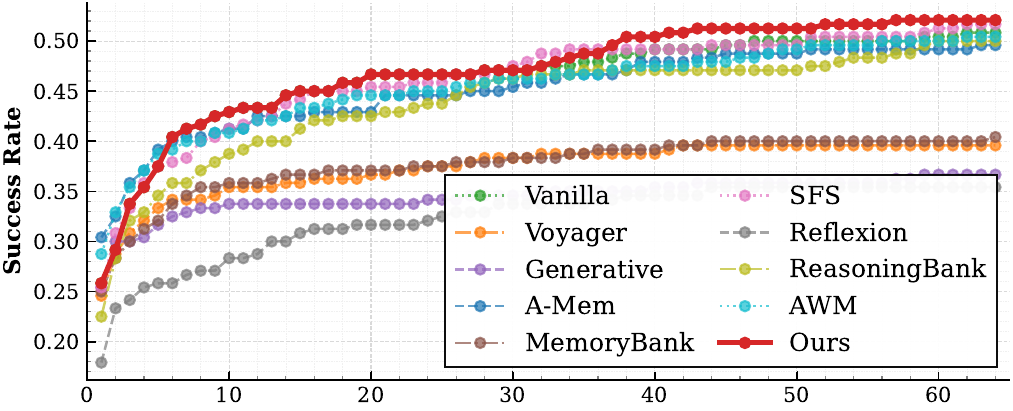}
        \caption{Success Rate of Tool-Query.}
        \label{fig:success_ratio_avg_tool}
  \end{subfigure}
  \begin{subfigure}[h]{0.48\textwidth}
        \includegraphics[width=0.99\textwidth]{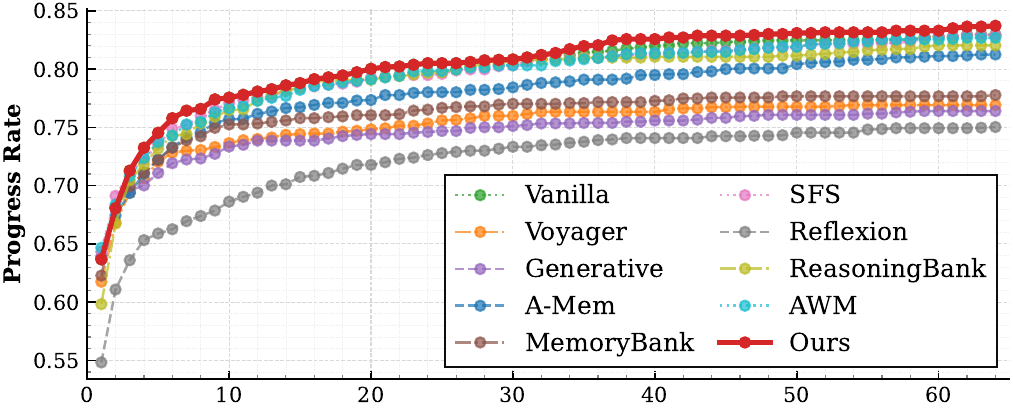}
        \caption{Progress Rate of Tool-Query.}
        \label{fig:progress_ratio_avg_tool}
  \end{subfigure}  
\caption{The average results of four LLM backbones on all Benchmarks.}
\label{fig:inference_scaling_alfworld}
\end{figure*}

\end{document}